\documentclass[12pt, oneside]{article}
\usepackage[left=1in,right=1in,top=1in,bottom=1in]{geometry}
\usepackage{graphicx}
\usepackage{amssymb}
\usepackage{amsmath}
\usepackage{amsthm}
\usepackage{thmtools}
\usepackage{thm-restate}

\usepackage{booktabs}
\usepackage{authblk}
\makeatletter
\renewcommand\AB@authnote[1]{\rlap{\textsuperscript{\normalfont#1}}}

\makeatother
\usepackage{color}
\usepackage[allcolors=blue,colorlinks=true]{hyperref}

\usepackage[utf8]{inputenc}
\usepackage[natbib=true,
            bibstyle=authoryear,
            citestyle=authoryear,
            uniquename=false,
            uniquelist=false,
            maxcitenames=2,
            maxbibnames=99]{biblatex}

\renewbibmacro{in:}{}

\usepackage[parfill]{parskip}

\title{Dynamic Visualization and Fast Computation for Convex Clustering via Algorithmic Regularization}
\author[1]{Michael Weylandt}
\author[1]{John Nagorski}
\author[1,2,3,4]{Genevera I. Allen}
\affil[1]{Department of Statistics, Rice University}
\affil[2]{Department of Computer Science, Rice University}
\affil[3]{Department of Electrical and Computer Engineering, Rice University}
\affil[4]{Jan and Dan Neurological Research Institute, Baylor College of Medicine}
\usepackage{xspace}
\newcommand{\carp}{{\tt CARP}\xspace}
\newcommand{\cbass}{{\tt CBASS}\xspace}
\newcommand{\carpviz}{{\tt CARP-VIZ}\xspace}
\newcommand{\cbassviz}{{\tt CBASS-VIZ}\xspace}
\newcommand{\clustRviz}{{\tt clustRviz}\xspace}
\newcommand{\cvxclustr}{{\tt cvxclustr}\xspace}
\newcommand{\cobra}{{\tt COBRA}\xspace}
\newcommand{\tcga}{\textsf{TCGA}\xspace}
\newcommand{\presidents}{\textsf{Presidents}\xspace}
\newcommand{\authors}{\textsf{Authors}\xspace}
\usepackage{subcaption}
\usepackage{bm}
\newcommand{\bE}{\bm{E}}
\newcommand{\bX}{\bm{X}}
\newcommand{\bU}{\bm{U}}
\newcommand{\bV}{\bm{V}}
\newcommand{\bZ}{\bm{Z}}
\newcommand{\bP}{\bm{P}}
\newcommand{\bQ}{\bm{Q}}
\newcommand{\bS}{\bm{S}}
\newcommand{\bT}{\bm{T}}
\newcommand{\bD}{\bm{D}}
\newcommand{\bL}{\bm{L}}
\newcommand{\bI}{\bm{I}}
\newcommand{\bw}{\bm{w}}
\newcommand{\bW}{\bm{\Upsilon}}

\newcommand{\br}{\bm{r}}
\newcommand{\bu}{\bm{u}}
\newcommand{\bp}{\bm{p}}
\newcommand{\bq}{\bm{q}}
\newcommand{\bv}{\bm{v}}
\newcommand{\bx}{\bm{x}}
\newcommand{\by}{\bm{y}}
\newcommand{\bz}{\bm{z}}

\usepackage{bbm}

\newcommand{\R}{\mathbb{R}}

\DeclareMathOperator{\prox}{\textsf{prox}}
\DeclareMathOperator{\vecop}{vec}
\DeclareMathOperator*{\argmin}{arg\,min}

\usepackage{enumerate}
\usepackage{algorithm}
\usepackage{setspace}

\newtheorem{lem}{Lemma}

\theoremstyle{remark}
\newtheorem*{rem}{Remark}

\date{Last Updated: \today}

\begin{document}
\maketitle
\begin{abstract}

Convex clustering is a promising new approach to the classical problem of clustering, combining strong performance in empirical studies with rigorous theoretical foundations. Despite these advantages, convex clustering has not been widely adopted, due to its computationally intensive nature and its lack of compelling visualizations. To address these impediments, we introduce \emph{Algorithmic Regularization}, an innovative technique for obtaining high-quality estimates of regularization paths using an iterative one-step approximation scheme. We justify our approach with a novel theoretical result, guaranteeing global convergence of the approximate path to the exact solution under easily-checked non-data-dependent assumptions. The application of algorithmic regularization to convex clustering yields the \textbf{C}onvex Clustering via \textbf{A}lgorithmic \textbf{R}egularization \textbf{P}aths (\carp) algorithm for computing the clustering solution path. On example data sets from genomics and text analysis, \carp delivers over a 100-fold speed-up over existing methods, while attaining a finer approximation grid than standard methods. Furthermore, \carp enables improved visualization of clustering solutions: the fine solution grid returned by \carp can be used to construct a convex clustering-based dendrogram, as well as forming the basis of a dynamic path-wise visualization based on modern web technologies. Our methods are implemented in the open-source \texttt{R} package \clustRviz, available at \url{https://github.com/DataSlingers/clustRviz}.

\end{abstract}
\emph{Keywords}: Clustering, Convex Clustering, Optimization, Algorithmic Regularization, Visualization, Dendrograms
\clearpage
\doublespace
\begin{refsection}
\section{Introduction}
Clustering, the task of identifying meaningful sub-populations in unlabelled data, is a fundamental problem in applied statistics, with applications as varied as cancer subtyping, market segmentation, and topic modeling of text documents. A wide range of methods for clustering have been proposed and we do not attempt to make a full accounting here, instead referring the reader to the recent book of \citet{Hennig:2015}. Perhaps the most popular clustering method, however, is hierarchical clustering \citep{Ward:1963}. Hierarchical clustering derives its popularity from an intuitive formulation, efficient computation, and powerful visualizations. Dendrogram plots, which display the family of clustering solutions simultaneously, provide an easily-understood summary of the global structure of the data, allowing the analyst to visually examine the nested group structure of the data. Despite its popularity, hierarchical clustering has several limitations: it is highly sensitive to the choice of distance metric and linkage used; it is a heuristic algorithm which lacks optimality guarantees; and the conditions under which hierarchical clustering recovers the true clustering are unknown. 

To address these limitations, several authors have recently studied a convex formulation of clustering \citep{Pelckmans:2005,Hocking:2011,Lindsten:2011,Chi:2015}. This convex formulation guarantees global optimality of the clustering solution and allows analysis of its theoretical properties \citep{Tan:2015,Zhu:2014,Radchenko:2017,Chi:2018}. Despite these advantages, convex clustering has not yet achieved widespread popularity, due to its computationally intensive nature and lack of dendrogram-based visualizations. In this paper, we address these problems with a efficient algorithm for computing convex clustering solutions with sufficient precision to construct interpretable accurate dendrograms and dynamic path-wise visualizations, thereby  making convex clustering a practical tool for applied data analysis. 

Our main theoretical contribution is the concept of Algorithmic Regularization, a novel computationally efficient approach for obtaining accurate approximations of regularization paths. We provide a theoretical justification for our proposed approach, showing that we can obtain a high-quality approximation simultaneously at all values of the regularization parameter. While we focus on the convex clustering problem, our proposed approach can be applied to a much wider range of problems arising in statistical learning. 

Using algorithmic regularization, we make two methodological contributions related to clustering: first, we propose an efficient algorithm, \carp, for computing convex clustering solutions. \carp is typically over one-hundred times faster than existing approaches, while simultaneously computing a much finer set of solutions than commonly used in practice. Secondly, we propose new visualization strategies for convex clustering based on \carp: a new dendrogram construction based on convex clustering paths and a novel ``path-wise'' visualization, which provides more information about the structure of the estimated clusters. We hope that, thanks to our proposed computational and visualization strategies, convex clustering will become a viable tool for exploratory data analysis. \carp and our proposed visualizations are implemented in our \clustRviz \texttt{R} package, available at  \url{https://github.com/DataSlingers/clustRviz}. 

The remainder of this paper is organized as follows: Section \ref{sec:dendro} reviews convex clustering in more detail and discusses the difficulties entailed in producing dendrograms from convex clustering. Section \ref{sec:carp} introduces the concept of ``Algorithmic Regularization,'' uses it to develop the \carp clustering algorithm, and gives theoretical guarantees of global convergence. Section \ref{sec:comparisons} compares \carp with existing approaches for convex clustering, demonstrating its impressive computational and statistical performance on several data sets. Section \ref{sec:visualization} describes several novel visualizations made possible by the \carp algorithms in the context of an extended text-mining example. Finally, Section \ref{sec:conclusion} concludes the paper with a discussion and proposes possible future directions for investigation.

\section{Convex Clustering and Dendrograms}
\label{sec:dendro}
We seek to represent the convex clustering solution path as a dendrogram, and in this section, discuss both the theoretical conditions and computational considerations for this task.  We first review the basic properties of convex clustering in Section \ref{sec:cclust} and then discuss dendrogram construction from convex clustering in Section \ref{sec:dendro_construction}.

\subsection{Convex Clustering}
\label{sec:cclust}

Let $\bX \in \R^{n \times p}$ denote a data matrix, consisting of $n$ observations
(rows of the matrix) in $p$ dimensions. The convex clustering problem, first discussed
by \citet{Pelckmans:2005} and later explored by \citet{Hocking:2011} and \citet{Lindsten:2011}, is a convex relaxation of the general clustering problem:
\[\argmin_{\bU \in \R^{n \times p}} \frac{1}{2}\|\bX - \bU\|_F^2 \quad  \text{  subject to  } \quad \sum_{1 \leq i < j \leq n} \mathbbm{1}_{\bU_{i\cdot} \neq \bU_{j\cdot}} \leq t,\]
where $\bU_{i\cdot}$ is the $i^{\text{th}}$ row of $\bU$.
Replacing the non-convex indicator function with an $\ell_q$-norm of the difference, we obtain the convex clustering problem:
\begin{equation}
\hat{\bU}_{\lambda} = \argmin_{\bU \in \R^{n \times p}} \frac{1}{2} \|\bX - \bU\|_F^2 + \lambda \left(\sum_{1 \leq i < j \leq n}  w_{ij} \|\bU_{i\cdot} - \bU_{j\cdot}\|_q\right).\label{eqn:cclust}
\end{equation}
Note that we have included non-negative fusion weights $\{w_{ij}\}$ in our convex relaxation. We will say more about the computational and statistical roles played by these weights below. 

The squared Frobenius norm loss function favors solutions which minimize the Euclidean distance between observations and their estimated centroids, while the fusion penalty term encourages the differences in columns of $\hat{\bU}_{\lambda}$ to be shrunk to zero. We interpret the solution $\hat{\bU}_{\lambda}$ as a matrix of cluster centroids, where each observation $\bX_{i\cdot}$ belongs to a cluster with centroid $(\hat{\bU}_{\lambda})_{i\cdot}$. For sufficiently large values of $\lambda$, the columns of $\hat{\bU}_{\lambda}$ will be shrunk together by the penalty term. We say that points with the same centroid belong to the same cluster; that is, the observations $\bX_{i\cdot}$ and $\bX_{j\cdot}$ are assigned to the same cluster if $(\hat{\bU}_{\lambda})_{i\cdot} = (\hat{\bU}_{\lambda})_{j\cdot}$

A major advantage of convex clustering is that the solution $\hat{\bU}_{\lambda}$ smoothly interpolates between clustering solutions, yielding a continuous path of solutions indexed by $\lambda$. At $\lambda = 0$, $\hat{\bU}_{\lambda = 0} = \bX$, resulting in a solution of $n$ distinct clusters, with each observation as the centroid of its own cluster. As $\lambda$ is increased, the fusion penalty encourages the columns of $\hat{\bU}_{\lambda}$ to merge together, inducing a clustering behavior. Finally, when $\lambda$ is large, all columns of $\hat{\bU}_{\lambda}$ are fully merged, yielding a single cluster centroid equal to the grand mean of the columns of $\bX$. Thus, the penalty parameter $\lambda$ determines both the number of clusters and the cluster assignments.

The choice of the fusion weights $\bw \in \R^{\binom{n}{2}}_{\geq 0}$ has a large effect on the statistical accuracy and computational efficiency of convex clustering. When uniform weights are used, convex clustering has a close connection to single-linkage hierarchical clustering, as shown by \citet{Tan:2015}. More commonly, weights inversely proportional to the distances between observations are used and have been empirically demonstrated to yield superior performance \citep{Hocking:2011,Chi:2015,Chi:2017}. Furthermore, setting many of the weights to zero dramatically reduces the computational cost associated with computing the convex clustering solution. We typically prefer using the rotation-invariant $\ell_2$-norm in the fusion penalty ($q = 2$), but one could employ $\ell_1$ or $\ell_{\infty}$ norms as well. 

By formulating clustering as a convex problem, it becomes possible to analyze its theoretical properties using standard techniques from the high-dimensional statistics literature. \citet{Tan:2015} show a form of prediction consistency and derive an unbiased estimator of the effective degrees of freedom associated with the solution. \citet{Zhu:2014} give sufficient conditions for exact cluster recovery at a fixed value of $\lambda$ (``sparsistency''). Like us, \citet{Radchenko:2017} are interested in properties of the entire solution path and give conditions under which convex clustering solutions \eqref{eqn:cclust} asymptotically yield the true dendrogram. 

\subsection{Constructing Dendrograms from Convex Clustering Paths}
\label{sec:dendro_construction}

In this paper, we propose to represent the convex clustering solution path as a dendrogram, an example of which is shown in Figure~\ref{fig:viz_dendro}. The convex clustering dendrogram is interpreted in much the same way as the classical dendrogram. As individual observations or groups of observations are fused together by the fusion penalty, they are denoted by merges in the tree structure.  The height of the merge in the tree structure is given by the value of the regularization parameter, $\lambda$, or more precisely $\log(\lambda)$, at which the fusion occurred. Thus, observations that fuse at small values of $\lambda$ are denoted by merges at the bottom of the dendrogram structure. As with hierarchical clustering, one can cut the dendrogram horizontally at a specific height to yield the associated clusters, and we can interpret the tree height at which merges occur as indicative of the similarity between groups. Before proceeding, however, it is natural to ask whether it is even possible to represent the convex clustering solution path as a dendrogram. This simple question turns out to have a rather subtle answer.

\begin{figure}
  \centering
  \includegraphics[width=4in]{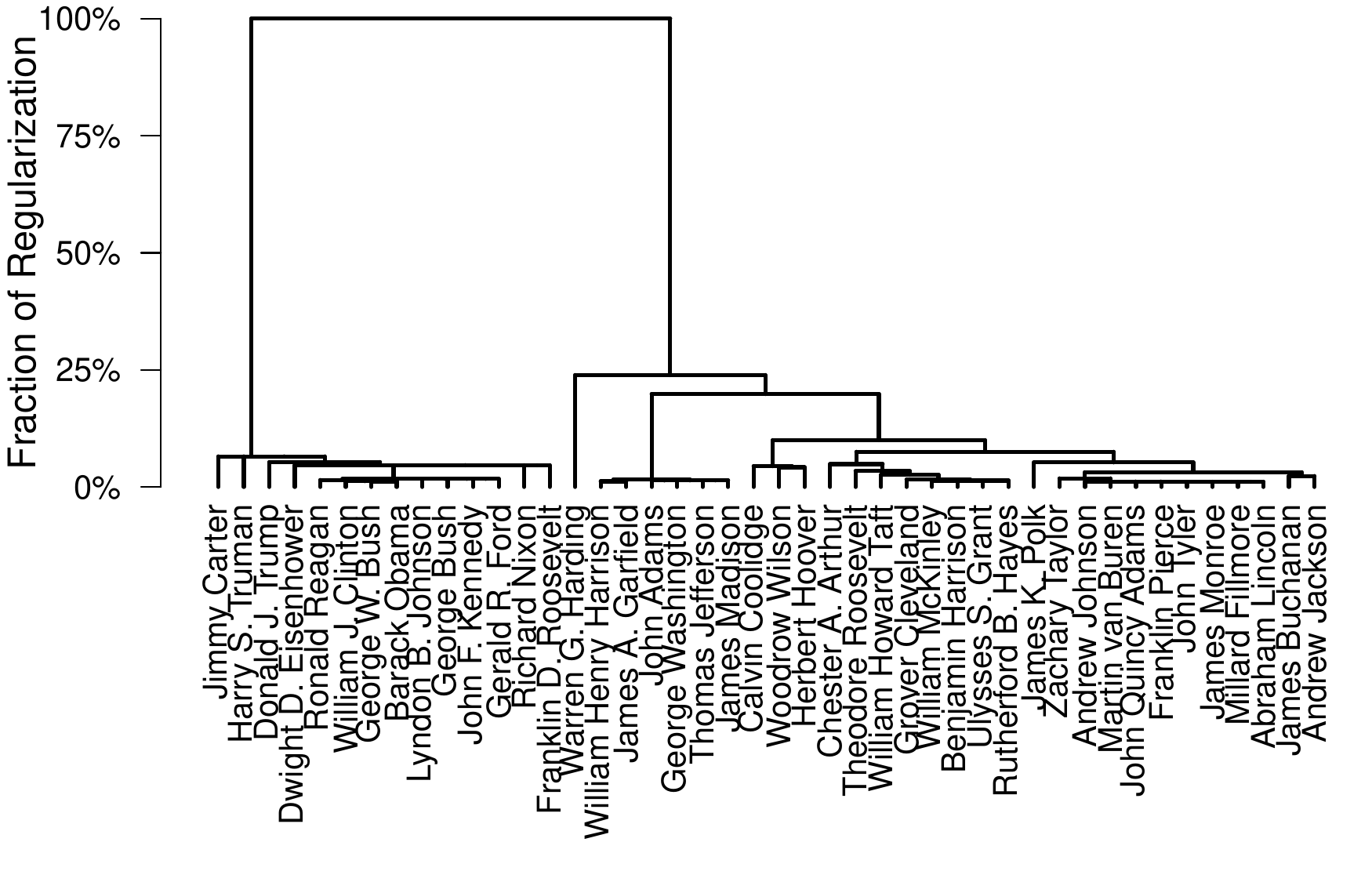}
  \caption{ A convex clustering dendrogram, displaying the 44 U.S. presidents. The interpretation of this dendrogram is discussed in more detail in Section \ref{sec:visualization}.}
  \label{fig:viz_dendro}
\end{figure}

There are two possible impediments to finding the desired dendrogram  representation: i) it may be impossible to represent the exact solution path as a dendrogram; and ii) it may be unrealistic to compute the solution path with enough precision to form a dendrogram. We first consider the question of whether the exact solution path admits a dendrogram representation. It is easy to observe that the exact solution path can only be written as a dendrogram if it is agglomerative, \emph{i.e.}, if the solution path consists of only fusions and no fissions. \citet{Hocking:2011} showed that if an $\ell_1$-norm is used for the fusion penalty and the weights are uniform then the solution path is  agglomerative, but their analysis does not generalize to arbitrary norms or arbitrary weight schemes. \citet{Chi:2018} showed that it is possible to select weights to yield a  agglomerative path, but their analysis applies only to a specific weighting scheme. In general, for an arbitrary data-driven  choice of weights, there is no theoretical guarantee that the solution path will be  agglomerative. In our experience, however, the solution path is agglomerative except in pathological situations.

Assuming that the exact solution path is agglomerative, we still must determine whether the solution path, \emph{as calculated}, is  agglomerative. While, in theory, this poses no additional challenge as the solution path is a continuous function of $\lambda$, in practice this poses a nearly insurmountable computational challenge. To construct a dendrogram, we require the exact values of $\lambda$ at which each fusion occurs. Since these values of $\lambda$ are not known \emph{a priori}, we are faced with a double burden: we must identify a critical set of values of $\lambda$ and solve the convex clustering problem \eqref{eqn:cclust} at those values. There are two widely-used approaches to finding the critical set of $\lambda$'s, path-wise algorithms and grid search, but, as we will show below, neither approach is sufficient in this case. 

Path-wise algorithms, such as those proposed for the Lasso \citep{Osborne:2000, Efron:2004} or generalized Lasso \citep{Tibshirani:2011} problems, compute the entire solution path exactly, identifying each value of $\lambda$ at which a sparsity event, equivalent to a fusion in our case, occurs. These algorithms are typically based on piece-wise linearity of the underlying solution path, which allows for smooth interpolation between sparsity events. \citet{Rosset:2007} studied the conditions under which solution paths are piece-wise linear and hence under which path-wise algorithms can be developed. It is easy to verify that these conditions do not hold with our preferred $\ell_2$-norm fusion penalty, so the solution path is not piece-wise linear and hence a path-wise algorithm cannot be employed for the convex clustering problem \eqref{eqn:cclust}. We note that several authors have proposed path-wise algorithms which can theoretically handle non-piece-wise-linear paths, but require solving an ordinary differential equation exactly \citep{Wu:2011,Zhou:2014,Xiao:2015}; in practice, these methods are computationally intensive and only approximate the path at a series of grid points, similar to the iterative methods we discuss next. For the $\ell_1$-norm fusion case, it is possible to apply path-wise algorithms for weights with specific graphical structures (see the discussion in \citet{Hocking:2011} or the examples considered in \citet{Tibshirani:2011})), but not for arbitrary graph structures with arbitrary weights, again eliminating the possibility of a general-purpose path-wise algorithm. 

Since there exists no exact path-wise algorithm, we might instead compute the convex clustering solution at a series of discrete points corresponding to a regular grid of $\lambda$'s. As we will show in Section~\ref{sec:comparisons}, however, this strategy is still computationally burdensome even using state-of-the-art algorithms and warm-start techniques. For example, the fastest algorithm considered by \citet{Chi:2015}, an Accelerated Alternating Minimization Algorithm, takes 6.87 hours and 19.81 hours to compute the solution path at 100 and 1000 regularly spaced $\lambda$'s, respectively, on a relatively small data set of dimension $n=438$ and $p = 353$. Furthermore, a grid of 100 or 1000 $\lambda$'s does not give us the value of $\lambda$ at which each fusion event occurs, which we need to construct a dendrogram.  Even if one wants to construct a dendrogram using only the order in which each fusion event occurs and their associated approximate values of $\lambda$, computing the path along a grid of 100 or 1000 $\lambda$'s only uniquely resolves 11.4\% or 37.44\% respectively of the fusion events needed to construct a dendrogram.  (See Table~\ref{tab:fusions} in Section~\ref{sec:comparisons} for complete results.)

In general, the computational cost of performing convex clustering is so high that it precludes its use as a practical tool for clustering and exploratory data analysis.  Further, computing the entire convex clustering solution path with fine enough precision to construct a dendrogram is an all but insurmountable task given existing computational algorithms for convex clustering. We seek to address this problem in this paper, using a novel computational technique which provides a fine grid of high-quality estimates of the regularization path. We introduce our approach and the clustering algorithm it suggests, \carp, in the next section.

\section{\carp: Convex Clustering via Algorithmic Regularization Paths}
\label{sec:carp}

We now turn our attention to efficiently computing solutions to the convex clustering problem \eqref{eqn:cclust} for a fine grid of $\lambda$, with a goal of dendrogram construction. Like many problems in the ``loss + penalty'' form, the convex clustering problem is particularly amenable to operator-splitting schemes such as the Alternating Direction Method of Multipliers (ADMM) \citep{Boyd:2011}. In statistical learning, we are often interested in the solution to a regularized estimation problem at a large number of values of $\lambda$. In this context, the performance of ADMMs is further improved by the use of ``warm-starts:'' if the ADMM is initialized near the solution, usually the solution at the previous value of $\lambda$, only a few iterations are typically required to obtain a solution which is accurate up to the statistical uncertainty inherent in the problem.

We propose an extreme version of this approach which we call \emph{Algorithmic Regularization}. Instead of running the ADMM to convergence, we take only a single ADMM step, after which we move to the next value of $\lambda$. By taking only a single ADMM step, we can significantly reduce the computational cost associated with estimating a regularization path. We can then use these computational savings to solve for a much finer grid of $\lambda$'s than we would typically use if employing a standard scheme. In essence, Algorithmic Regularization allows us to exchange computing an exact solution for a small set of $\lambda$'s for calculating a highly accurate approximation at a large set of $\lambda$'s. Usefully, we can now use a $\lambda$ grid with sufficiently fine resolution that we can fully capture the desired dendrogram structure in a reasonable amount of time.

\citet{Chi:2015} first considered the use of the ADMM to solve the convex clustering problem \eqref{eqn:cclust}. To apply the ADMM, we introduce $\bD$, the directed difference matrix used to calculate to the pairwise differences of rows of $\bU$, and an auxiliary variable $\bV$, corresponding to the matrix of between-observation differences:
\[\argmin_{\substack{\bU \in \R^{n \times p}\\\bV \in \R^{\binom{n}{2}\times p}}} \frac{1}{2}\|\bX - \bU\|_F^2 + \lambda\underbrace{\sum_{k = 1}^{\binom{n}{2}} w_k \|\bV_{k\cdot}\|_q}_{P(\bV; \bw, q)} \text{ subject to } \bV = \bD\bU.\] Applying the ADMM with warm-starts to the above, we obtain the following algorithm: 
\begin{algorithm}[H]
Initialize $l = 0$, $\lambda_l = \epsilon$, $\bV^{(0)} = \bZ^{(0)} = \bD\bX$\\
Repeat until $\|\bV^{(k)}\| = 0$: 
\begin{itemize}
\item Repeat until convergence:
\begin{enumerate}
\item[(i)] $\bU^{(k + 1)} = \bL^{-T}\bL^{-1}\left(\bX + \bD^T(\bV^{(k)} - \bZ^{(k)})\right)$
\item[(ii)] $\bV^{(k+1)} = \prox_{\lambda_l P(\cdot; \bw, q)}\left(\bD \bU^{(k + 1)} + \bZ^{(k)}\right)$
\item[(iii)] $\bZ^{(k + 1)} = \bZ^{(k)} + \bD \bU^{(k + 1)} - \bV^{(k+1)}$
\item[(iv)] $k := k + 1$
\end{enumerate}
\item Store $\hat{\bU}_{\lambda_l} = \bU^{(k)}$
\item Update regularization: $l := l + 1$; $\lambda_l := \lambda_{l - 1} * t$
\end{itemize} 
Return $\{\hat{\bU}_{\lambda}\}$ as the regularization path
\caption{Warm-Started ADMM for the Convex Clustering Problem \eqref{eqn:cclust}}
\label{alg:short_admm}
\end{algorithm}\vspace{-0.3in}
where $\bZ$ is the dual variable with the same dimensions as $\bV$, $\bL$ is Cholesky factorization of $\bI + \bD^T\bD$, and $\prox_{f(\cdot)}(\bx) = \argmin_{\bz} \frac{1}{2}\|\bx - \bz\|_2^2 + f(\bz)$ is the proximal mapping of a general function $f$. Note that, if sparse weights are used, the corresponding rows of $\bD$, $\bV$, and $\bZ$ may be omitted, yielding more efficient updates. Additionally, note that, because we use a multiplicative update for $\lambda$, we must initialize at $\lambda = \epsilon$, for some small $\epsilon$, rather than at $\lambda = 0$. A derivation and more detailed statement of this algorithm are given in Section \ref{app:derivations} of the Supplementary Materials.

We now take Algorithm \ref{alg:short_admm} as the basis for our extreme early stopping strategy of Algorithmic Regularization. Removing the the inner loop, we obtain the following scheme, which we refer to as \carp--\textbf{C}onvex Clustering via \textbf{A}lgorithmic \textbf{R}egularization \textbf{P}aths:
\begin{algorithm}[H]
Initialize $k = 0$, $\gamma^{(k)} = \epsilon$, $\bV^{(0)} = \bZ^{(0)} = \bD\bX$\\
Repeat until $\|\bV^{(k)}\| = 0$: 
\begin{enumerate}
\item[(i)] $\bU^{(k + 1)} = \bL^{-T}\bL^{-1}\left(\bX + \bD^T(\bV^{(k)} - \bZ^{(k)})\right)$
\item[(ii)] $\bV^{(k+1)} = \prox_{\gamma^{(k)} P(\cdot; \bw, q)}\left(\bD \bU^{(k + 1)} + \bZ^{(k)}\right)$
\item[(iii)] $\bZ^{(k + 1)} = \bZ^{(k)} + \bD \bU^{(k + 1)} - \bV^{(k+1)}$
\item[(iv)] $k := k + 1$, $\gamma^{(k)} = \gamma^{(k - 1)} * t$
\end{enumerate}
Return $\{\bU^{(k)}\}$ as the \carp path.\vspace{0.1in}
\caption{\carp: Algorithmic Regularization for the Convex Clustering Problem \eqref{eqn:cclust}}
\label{alg:short_carp}
\end{algorithm}\vspace{-0.3in}
The fundamental difference between Algorithm \ref{alg:short_admm} and Algorithm \ref{alg:short_carp} is that Algorithm \ref{alg:short_carp} does not have an ``inner loop'' in which ADMM iterates are repeated until convergence to the exact solution for a fixed value of the regularization parameter. As such, the \carp iterates $\{\bU^{(k)}\}$ are not exact solutions to the convex clustering problem \eqref{eqn:cclust} for any value of $\lambda$, though they are typically accurate approximations in a sense that Theorem \ref{thm:hausdorff} below makes precise. Our notation reflects this distinction and replaces $\lambda_l$ with $\gamma^{(k)}$ in the $\bV$-update to avoid suggesting any false equivalence. A more detailed formulation of the \carp algorithm is given in Section \ref{app:derivations} of the Supplementary Materials. 

The role of the step-size parameter $t$ in Algorithm \ref{alg:short_carp} is particularly important in understanding \carp. The step size $t$ controls the fineness of the $\{\gamma^{(k)}\}$ grid used internally by \carp and, as such, serves as a \emph{computational} tuning parameter controlling how well the \carp path approximates the true convex clustering path. Decreasing $t$ has benefits for both \emph{local} and \emph{global} accuracy of the \carp path: a smaller value of $t$ yields an approximate solution path which has a finer set of grid points $\{\gamma^{(k)}\}$ (improved global accuracy) and more accurate approximations $\{\bU^{(k)}\}$ at each of those grid points (improved local accuracy). This is in contrast to standard approaches where the user has to pre-specify the $\{\lambda_l\}$ grid used and the stopping tolerance of the iterative algorithm to strike a balance between local accuracy and global accuracy obtainable in a given amount of time. As we will see in Section \ref{sec:comparisons}, the Algorithmic Regularization strategy of replacing an iterative algorithm with a one-step approximation thereof allows us to improve both local and global accuracy at a fraction of the cost of competing methods. 

While this may all seem rather fishy, the following theorem shows that, in the limit of small changes to the regularization level (\emph{i.e.}, $(t, \epsilon) \to (1, 0)$), there is indeed no loss in accuracy induced by the one-step approximation. In fact, we are able to show a very strong form of convergence, so-called \emph{Hausdorff} convergence, in both the primal and dual variables. Hausdorff convergence implies two different convergence results hold simultaneously for both the primal and dual variables.  The first, $\sup_{\lambda} \inf_k \left\|\bU^{(k)} - \hat{\bU}_{\lambda}\right\| \to 0$, implies every convex clustering solution will be recovered by \carp as $(t,\epsilon) \to (1, 0)$. The second, $\sup_{k} \inf_{\lambda} \left\|\bU^{(k)} - \hat{\bU}_{\lambda}\right\| \to 0$, implies that any clustering produced by  \carp as $(t, \epsilon) \to (1, 0)$ is a valid convex clustering solution for some $\lambda$. More memorably, Theorem \ref{thm:hausdorff} shows that asymptotically \carp produces ``the whole regularization path and nothing but the regularization path:''

\vspace{0.1in}\begin{restatable}{thm}{hausdorff} \label{thm:hausdorff}
As $(t, \epsilon) \to (1, 0)$, where $t$ is the multiplicative step-size update and $\epsilon$ is the initial regularization level, the primal and dual \carp paths converge to the primal and dual convex clustering paths in the Hausdorff metric: that is,
\begin{align*}
d_H(\{\bU^{(k)}\}, \{\hat{\bU}_{\lambda}\}) \equiv \max\left\{\sup_{\lambda} \inf_k \left\|\bU^{(k)} - \hat{\bU}_{\lambda}\right\|, \sup_{k} \inf_\lambda \left\|\bU^{(k)} - \hat{\bU}_{\lambda}\right\|\right\}  \xrightarrow{(t, \epsilon) \to (1, 0)} 0\\
d_H(\{\bZ^{(k)}\}, \{\hat{\bZ}_{\lambda}\}) \equiv \max\left\{\sup_{\lambda} \inf_k \left\|\bZ^{(k)} - \hat{\bZ}_{\lambda}\right\|, \sup_{k} \inf_\lambda \left\|\bZ^{(k)} - \hat{\bZ}_{\lambda}\right\|\right\} \xrightarrow{(t, \epsilon) \to (1, 0)} 0
\end{align*}
where $\bU^{(k)}, \bZ^{(k)}$ are the values of the $k^{\textrm{th}}$ \carp iterate and $\hat{\bU}_{\lambda}, \hat{\bZ}_{\lambda}$ are the exact solutions to the convex clustering problem \eqref{eqn:cclust} and its dual at $\lambda$.
\end{restatable}

A full proof of Theorem \ref{thm:hausdorff} is given in Section \ref{app:proof} of the Supplementary Materials, but we highlight the three essential elements here: i) we obtain a high-quality initialization at the first step by setting $\bU^{(0)} = \bX$ which is the exact solution at $\lambda = 0$  ($\hat{\bU}_{\lambda = 0} = \bX$); ii) the convex clustering problem \eqref{eqn:cclust} is strongly convex due to the squared Frobenius norm loss, so the ADMM converges quickly (linearly); and iii) the solution path is Lipschitz as a function of $\lambda$, so $\hat{\bU}_{\lambda}$ does not vary too quickly.  Putting these together, we show that \carp can ``track'' the exact solution path closely, with the approximation error at each step decreasing at a faster rate than the exact solution changes. We emphasize that both strong convexity and Lipschitz solution paths are features of the optimization problem, not the specific data, and are easily checked in practice. A careful reading of the proof of Theorem \ref{thm:hausdorff} will reveal that our analysis applies to a much wider class of problems than convex clustering. We consider the application of algorithmic regularization to the closely related problem of convex bi-clustering \citep{Chi:2017} in Section \ref{sec:cbass} of the Supplementary Materials, where we develop the \cbass (\textbf{C}onvex \textbf{B}i-Clustering via \textbf{A}lgorithmic Regularization with \textbf{S}mall \textbf{S}teps) algorithm, but we leave examination of the more general phenomenon of Algorithmic Regularization to future work.

While Theorem \ref{thm:hausdorff} implies that a sufficiently small choice of step size $t$ allows for exact dendrogram recovery, in practice it is often challenging to select $t$ sufficiently small without requiring excessive computation. Instead, we take a small, but not infinitesimal, value of $t$ and add a back-tracking step to ensure that fusions necessary for dendrogram construction are exactly identified. We refer to the back-tracking version of \carp as \carpviz, for reasons which will be clarified in Section \ref{sec:comparisons}. Furthermore, a post-processing step can be used to find fusions that back-tracking is unable to isolate. Details of the back-tracking and post-processing rules, as implemented in \clustRviz, are given in Section \ref{app:backtrack} of the Supplementary Materials.

Algorithmic Regularization, as used here, was first discussed in \citet{Hu:2016} without theoretical justification and was successfully applied to unmixing problems in hyperspectral imaging by \citet{Drumetz:2017}. We emphasize that, while grounded in standard optimization techniques, algorithmic regularization takes a different perspective than other commonly-used computational approaches, more akin to function approximation than standard optimization. The algorithmic regularization perspective is principally concerned with recovering the overall structure of the solution path than with obtaining the most accurate solution possible at a fixed value of $\lambda$. As such, our Theorem \ref{thm:hausdorff} is of a different character than similar results appearing in the optimization literature, making a claim of global path-wise convergence rather than local point-wise convergence. This ``holistic'' viewpoint is necessary to recover dendrograms, a major goal of this paper, but has also recently been found useful for choosing tuning parameters in penalized regression problems \citep{Chichignoud:2016}.  

\subsection{Related Work}
Several authors have considered path approximation algorithms not unlike \carp, often in the context of boosting algorithms. \citet{Rosset:2004}, \citet{Zhao:2007}, and \citet{Friedman:2012} all consider iterative algorithms which approximate solution paths of regularized estimators. Of these, the approach of \citet{Zhao:2007}, who consider a path approximation algorithm for the Lasso, is most similar to our own. Assuming strong convexity, their \texttt{BLasso} algorithm exactly recovers the lasso solution path as the step-size goes to zero. Their algorithm can be viewed as an application of algorithmic regularization to greedy coordinate descent with an additional back-tracking step to help isolate events of interest (variables entering or leaving the active set). Our algorithmic regularization strategy is simpler than their approach, as it does not require the back-tracking step, and can be applied to more general penalty functions.

\citet{Clarkson:2010} and \citet{Giesen:2012} consider the problem of obtaining approximate solutions for a set of parameterized problems subject to a simplex constraint, though their approach still requires running an optimization step until approximate convergence at each step, Building on this, \citet{Tibshirani:2015} proposes a general framework for constructing ``stagewise'' solution paths, which can be interpreted as an application of algorithmic regularization to the Frank-Wolfe algorithm \citep{Jaggi:2013}. He shows that the stagewise estimators achieve the optimal objective value as the step-size is taken to zero; if a strong convexity assumption is added, it is not difficult to extend his Theorem 2 to recover a result similar to our Theorem \ref{thm:hausdorff}. Our framework is more general than his, as we do not require require the gradient of the loss function to be Lipschitz and we admit more general regularizers.

\section{Numerical and Timing Comparisons} \label{sec:comparisons}

Having introduced \carp and given some theoretical justification for its use, we now consider its performance on representative data sets from text analysis and genomics. As we will show, \carp achieves the superior clustering performance of convex clustering at a small fraction of the computational cost. Throughout this section, we use two example data sets: \tcga and \authors. The \tcga data set ($n = 438$, $p = 353$) contains log-transformed Level III RPKM gene expression levels for 438 breast-cancer patients from \citet{TCGA:2012}. The \authors data set ($n=841, p = 69$) consists of word counts from texts written by four popular English-language authors (Austen, London, Shakespeare, and Milton). For all comparisons, we use \clustRviz's default sparse Gaussian kernel weighting scheme described in the package documentation. For timing comparisons, the Accelerated ADMM and AMA proposed by \citet{Chi:2015} and implemented in their \cvxclustr package was used; our \clustRviz package was used for \carp and \carpviz. All comparisons were run on a 2013 iMac with a 3.2 GHz Intel i5 processor and 16 GB of 1600 MHz DDR3 memory.

While Theorem \ref{thm:hausdorff} strictly only applies for asymptotically small values of $t$, \carp paths are high-quality approximations of the exact convex clustering solution path, even at moderate values of $t$. We assess accuracy of the \carp paths by considering the normalized relative Hausdorff distance between the primal \carp path ($\bU^{(k)}$) and the exact solution ($\hat{\bU}_{\lambda}$)
\[\widetilde{d_H}(\{\bU^{(k)}\}, \{\hat{\bU}_{\lambda}\}) = \frac{\max\left\{\sup_{\lambda} \inf_k \left\|\bU^{(k)} - \hat{\bU}_{\lambda}\right\|_2, \sup_{k} \inf_\lambda \left\|\bU^{(k)} - \hat{\bU}_{\lambda}\right\|_2\right\}}{n * p * \|\bD\bX\|_{2,\infty}}\] where $\|\cdot\|_{2, \infty}$ is the maximum of the $\ell_2$-norms of the rows of a matrix. (We include the normalization constants in the denominator so that our distance measure does not depend on the size or numerical scale of the data.) In order to calculate the Hausdorff distance, the \carp path with a very small step-size $t$ was used in lieu of the exact solution $\{\hat{\bU}_{\lambda}\}$. As can be seen in Figure \ref{fig:hausdorff}, the \carp path is highly accurate even at moderate values of $t$ and converges quickly to the exact solution as $t \to 1$. \carpviz, which uses an adaptive choice of $t$ to isolate each individual fusion, performs even better than the fixed step-size \carp, attaining a very accurate approximation of the true clustering path.

\begin{figure}[ht]
   \centering
   \includegraphics[width=4in,height=2.2in]{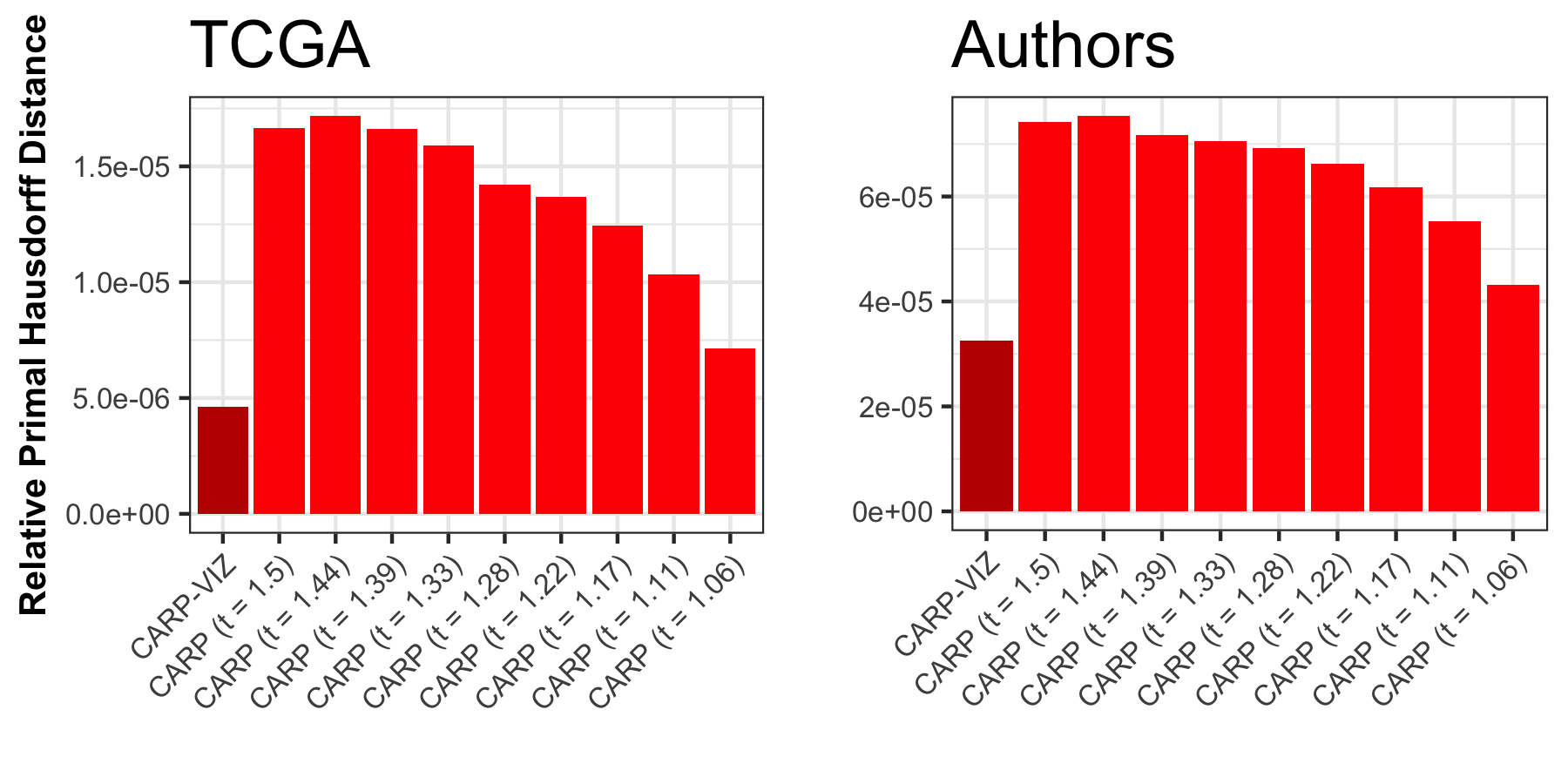}
    \caption{Normalized Relative Primal Hausdorff Distance between the exact convex clustering solution and \carp Paths for various values of $t$. As $t$ decreases, the \carp Paths converge to the exact convex clustering solution path, consistent with Theorem \ref{thm:hausdorff}.}
    \label{fig:hausdorff}
\end{figure}

Even though they are highly accurate, \carp paths are relatively cheap to compute. In Figure \ref{fig:timing}, we compare the computational cost of \carp with the algorithms proposed in \citet{Chi:2015}. As shown in Figure \ref{fig:timing}, \carp significantly outperforms the Accelerated AMA and ADMM algorithms. At large step-sizes ($t = 1.1, 1.05$), \carp terminates in less than a minute and, even at finer grid-sizes ($t = 1.01, 1.005$), \carp takes only a few minutes to run. The \carpviz variant takes significantly longer than standard \carp, though it still outperforms the AMA, taking about an hour for \tcga, rather than the six and a half hours required to solve the AMA at 100 grid points. This improvement in computational performance is even more remarkable when we note that \carp and \carpviz produce a fine grid of solutions by default: on the \tcga data, \carp with $t = 1.01$ produces over 2047 distinct grid points in under five minutes.

\begin{figure}[ht]
   \centering
   \includegraphics[width=4in,height=2in]{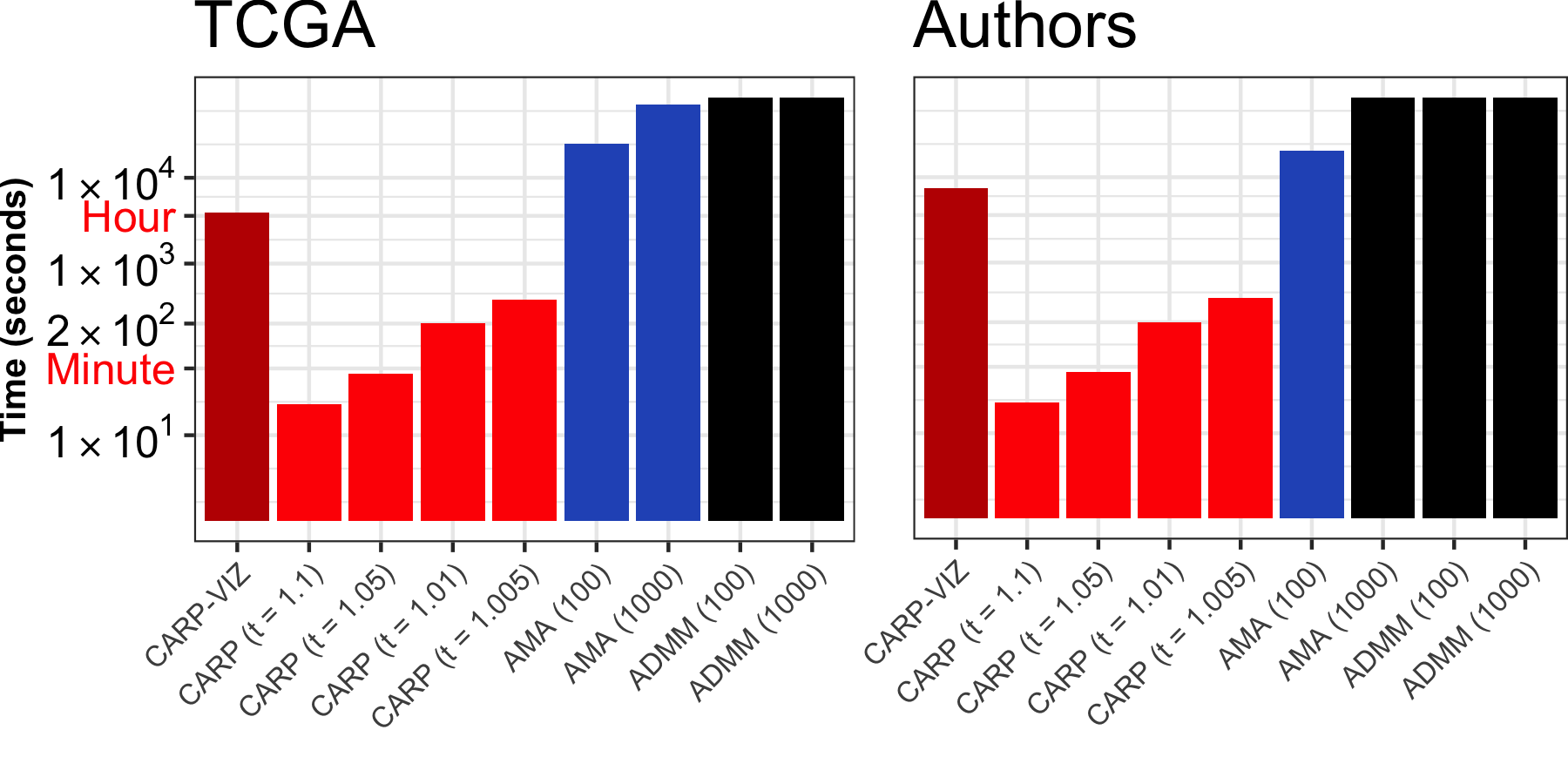}
    \caption{Time required to compute clustering solution path (logarithmic scale). \carp produces high-quality path approximations in a fraction of the time of standard iterative algorithms. Timings in black indicate calculations that took more than 24 hours to complete.}
    \label{fig:timing}
\end{figure}

The fine grid of solutions returned by \carp and \carpviz result in much improved dendrogram recovery, as measured by the fraction of  unique clustering assignments each method returns. Table \ref{tab:fusions} shows recovery results for \carp (at several values of $t$), \carpviz, and standard fixed-grid methods. It is clear that back-tracking employed by \carpviz is necessary for exact dendrogram recovery and that \carpviz should be used if the exact dendrogram recovery is required for visualization. Even with moderate step-sizes ($t = 1.1, 1.05$), however, \carp is still able to estimate the dendrogram far more accurately and more rapidly than standard iterative methods, making it a useful alternative for exploratory work. 

%% JN's results
\begin{table}[ht]
  \centering
  \begin{tabular}{c|cc}
  \toprule
  Method & \tcga $(n = 438, p = 353)$ & \authors $(n = 841, p = 69)$ \\
  \midrule
  \carp $(t = 1.1)$ & 5.93\% & 4.40\%  \\
  \carp $(t = 1.05)$ & 11.18\% & 8.09\%\\
  \carp $(t = 1.01)$ & 41.55\% & 22.71\%\\
  \carp $(t = 1.005)$ & 60.27\% & 30.56\%\\
  \midrule 
  \carpviz & \textbf{100\%} & \textbf{100\%}\\
  \midrule
  100-Point Grid & 11.42\% & 5.00\%\\
  1000-Point Grid & 37.44\% & --- \\
  \bottomrule
  \end{tabular}
  \caption{Proportion of dendrogram recovered by \carp, \carpviz, and standard fixed-grid methods. The back-tracking employed by \carpviz is necessary for exact dendrogram recovery, but fixed step-size \carp is still able to recover the dendrogram more accurately than standard fixed-grid approaches.}
  \label{tab:fusions}
\end{table}

On the \tcga data, \carp with $t = 1.05$ is able to recover the dendrogram with the same accuracy in a minute that the AMA attains in six and a half hours. With $t = 1.01$, \carp recovers the dendrogram more accurately in under five minutes than the AMA does in nineteen hours (1000 grid points). \carp achieves these improvements by using its computation efficiently: while a standard optimization algorithm may spend several hundred iterations at a single value of the regularization parameter, \carp only spends a single iteration. By reducing the number of iterations at each grid point, it can take examine a much finer grid in less time. This trade-off is particularly well-suited for our goal of dendrogram recovery, which requires a fine grid of solutions to assess the order of fusions, but does not depend on the exact values of the estimated centroids. While not a primary focus of this paper, the high-quality dendrogram estimation allowed by \carp translates into improved statistical performance as well. We compare the statistical performance of \carp with other clustering methods in Section \ref{app:additional_comparisons} of the Supplementary Materials.

\section{Visualization of \carp Results} \label{sec:visualization}

In this section, we discuss visualization of convex clustering results, emphasizing the role that \carp can play in exploratory data analysis. The visualizations illustrated in this section can all be produced using our \clustRviz \texttt{R} package. Throughout this section, we will use the \presidents data set, ($n = 44, p = 75$) which contains log-transformed word counts of the 75 most variable words taken from the aggregated major speeches (primarily Inaugural and State of the Union Addresses) of the 44 U.S. presidents through mid-2018. (We consolidate the two non-consecutive terms of Grover Cleveland.)

We begin by considering a dendrogram representation of this data, as shown in Figure \ref{fig:pres_dendro}. For each dendrogram, we have colored the observations by historical period: Founding Fathers, pre-Civil War, pre-World War II, and modern. Given the evolution of the English language and the changing political concerns of these periods, we would expect clustering methods to group the presidents according to historical period. With three exceptions, \carp clearly identifies the four historical periods, with the modern period being particularly well-separated. The performance of hierarchical clustering is highly sensitive to the choice of linkage: Ward's linkage \citep{Ward:1963} does almost as well as \carp, but does not clearly separate the pre-Civil War and pre-World War II periods. Single linkage correctly identifies the modern period, but otherwise does not separate the pre-modern presidents. Complete linkage performs the worst, clustering Donald Trump with the Founding Fathers,  Garfield, and Harrison instead of with other modern presidents. We note that Harrison is consistently misclustered by all methods considered: we believe this is due to the fact he died thirty-one days into his first term and did not leave a lengthy textual record. 

\begin{figure}
  \centering
  \begin{subfigure}[b]{.47\linewidth}
  \includegraphics[width=\linewidth]{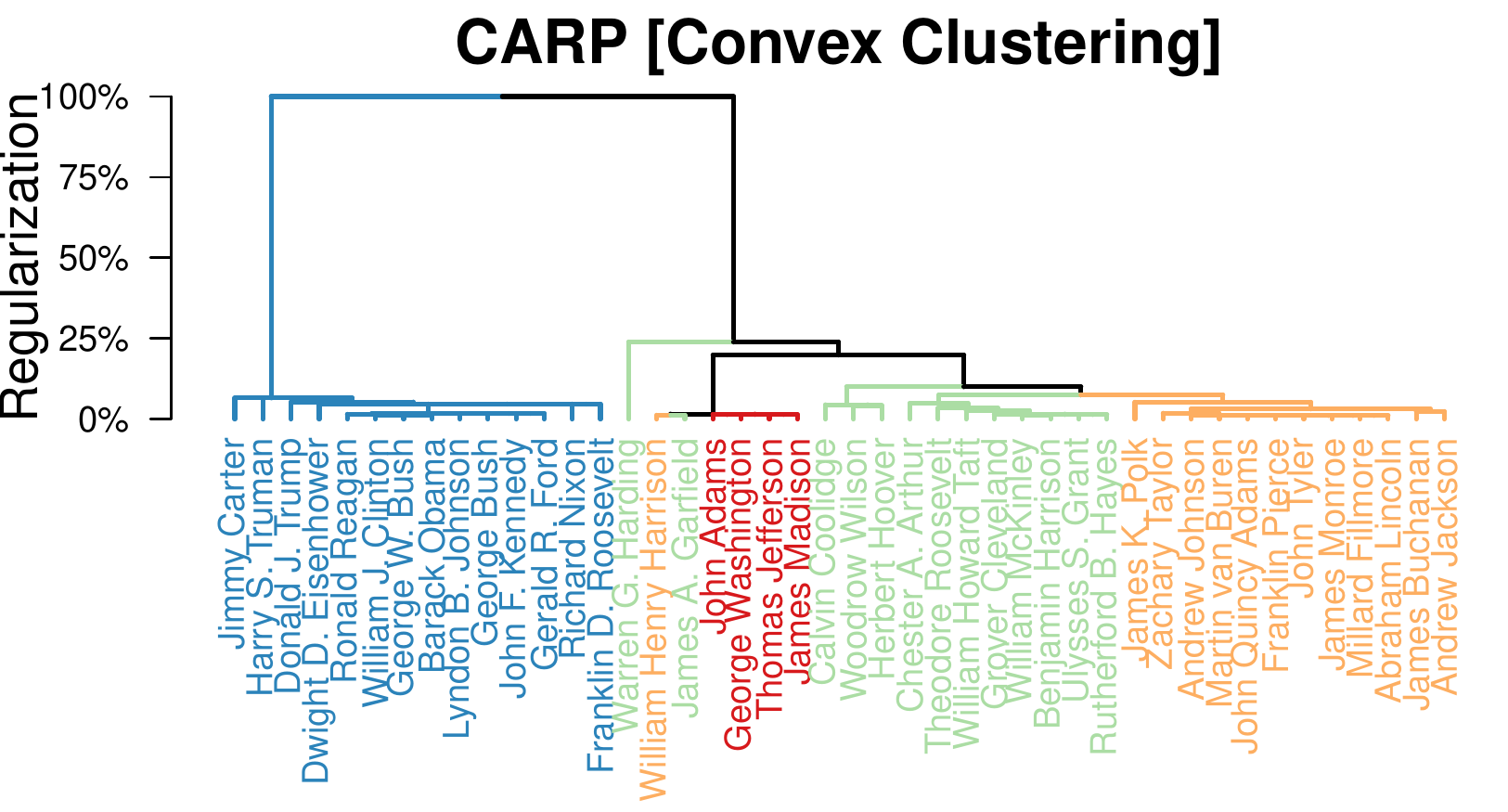}
  \end{subfigure}
  \begin{subfigure}[b]{.47\linewidth}
  \includegraphics[width=\linewidth]{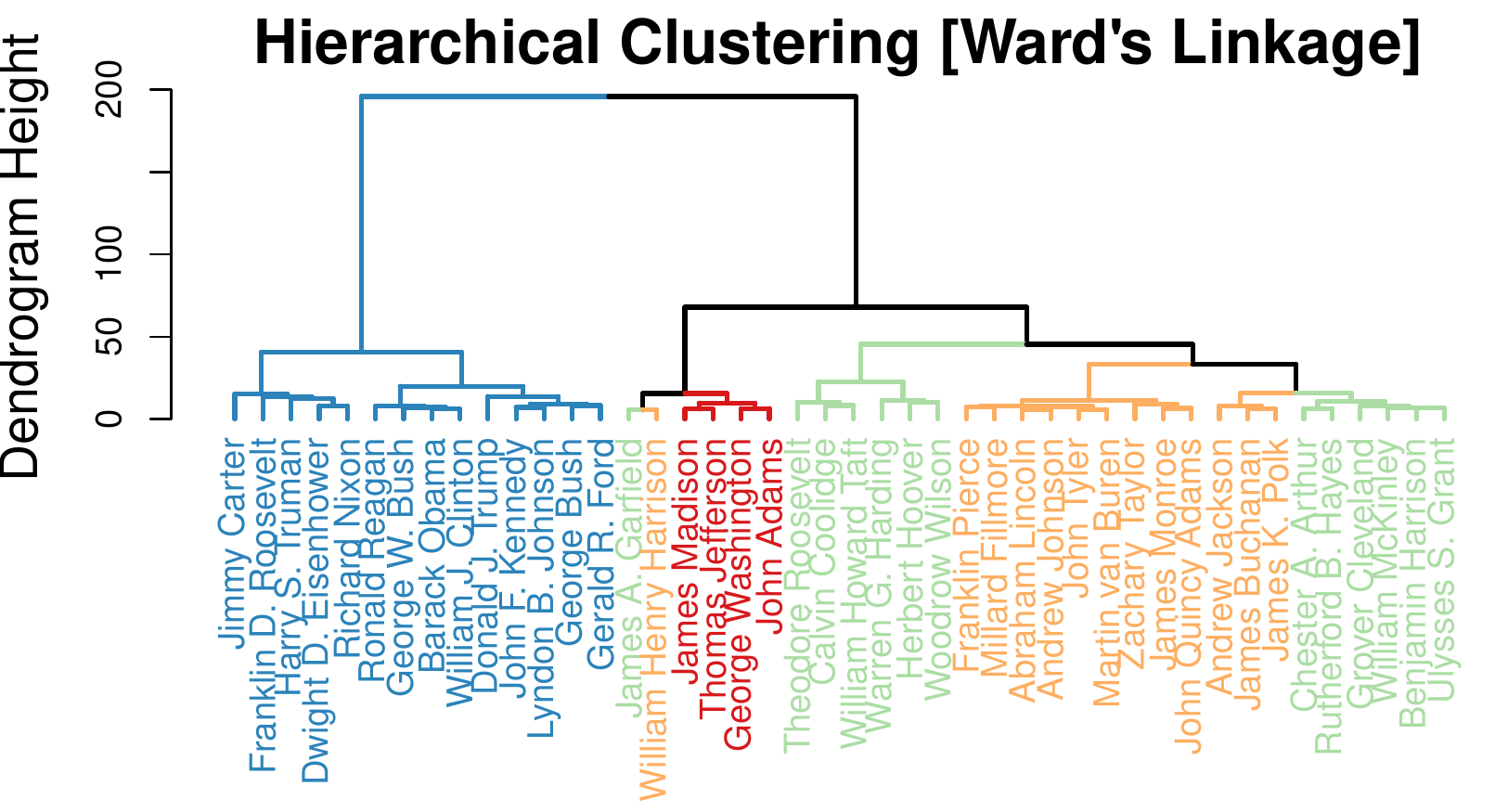}
  \end{subfigure}

  \begin{subfigure}[b]{.47\linewidth}
  \includegraphics[width=\linewidth]{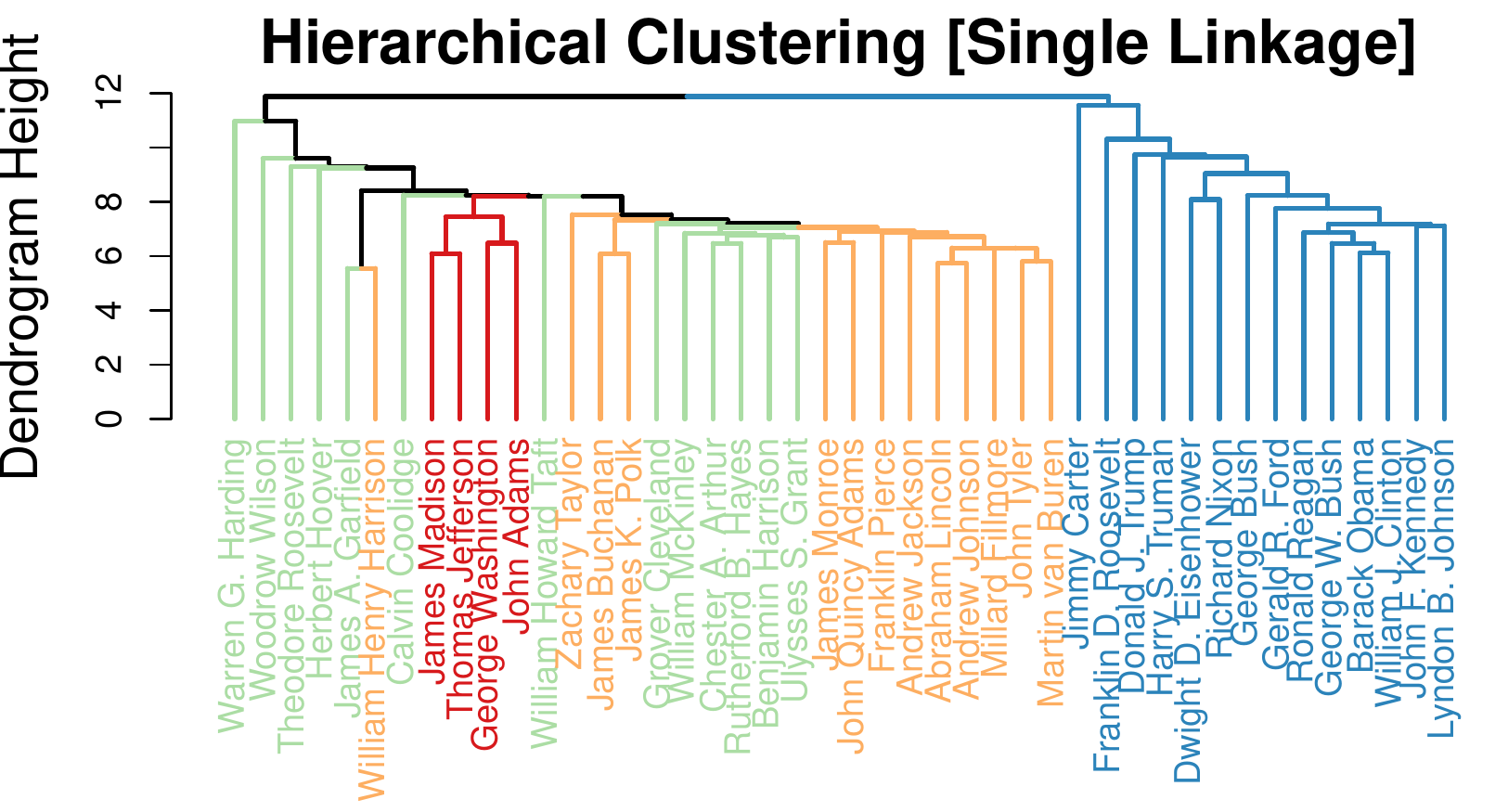}
  \end{subfigure}
  \begin{subfigure}[b]{.47\linewidth}
  \includegraphics[width=\linewidth]{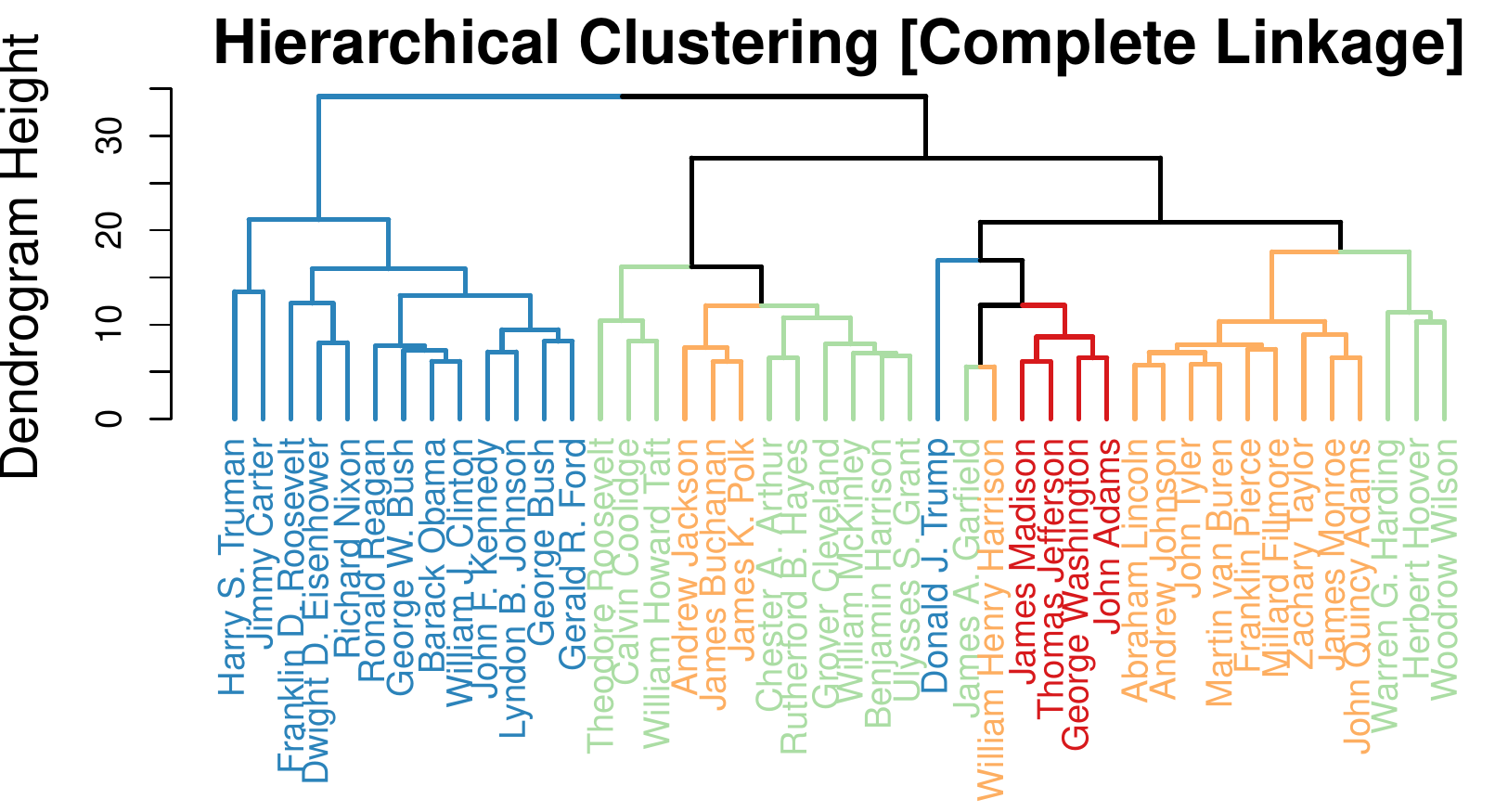}
  \end{subfigure}
  \caption{ Comparison of \carp (top left) and Euclidean-distance Hierarchical Clustering dendrograms on the \presidents data. The presidents are colored according to historical period: Founding Fathers (red, 1789-1817, Washington to  Madison, $n = 4$); pre-Civil War (orange, 1817-1869, Monroe to Johnson, $n = 13$); pre-World War II (green, 1869-1933, Grant to Hoover, $n=13$); and modern (blue, 1933-present, F.D.~Roosevelt to Trump, $n = 14$). We consider Johnson to be a pre-Civil War president as he ascended to the presidency following the assassination of Lincoln rather than being directly elected.}
  \label{fig:pres_dendro}
\end{figure}

Beyond allowing accurate dendrogram construction, the \carp paths are themselves interesting to visualize. By plotting the path traced by the \carp iterates $\bU^{(k)}$, we can observe exactly how \carp forms clusters from a given data set. For dimension reduction, we typically plot the projection of $\bU^{(k)}$ onto the principal components of $\bX$, though \clustRviz allows visualization of the raw features as well. Unlike the \carp-dendrogram, the path plot allows examination of the structure of the estimated clusters and not just their membership. By displaying the original observations on the path plot, we also enable comparison of the estimated centroids with the original data. Modern web technologies allow us to display these path plots dynamically, forming a movie with each \carp iterate as a separate frame. The fine solution grid returned by \carp is especially relevant here, enabling us to construct movies in which the observations move smoothly. We have found that the smoothness of the movie is a useful heuristic to assess whether a small enough step-size $t$ was used: if the paths ``jump'' conspicuously from one frame to the next, one should consider re-running \carp with a smaller step-size.

Figure \ref{fig:path_viz} shows three frames of such a movie: in each frame, on the left side, we see the Founding Fathers cluster being merged to the other pre-modern presidents, while on the right, we see a clear cluster of modern presidents. A closer examination of the central frame reveals additional information not visible in the \carp-dendrogram: Harding, the last president to join the pre-modern cluster, is an outlier lying between two clusters rather than far to one side.

\begin{figure}
  \centering
  \includegraphics[height=2.5in]{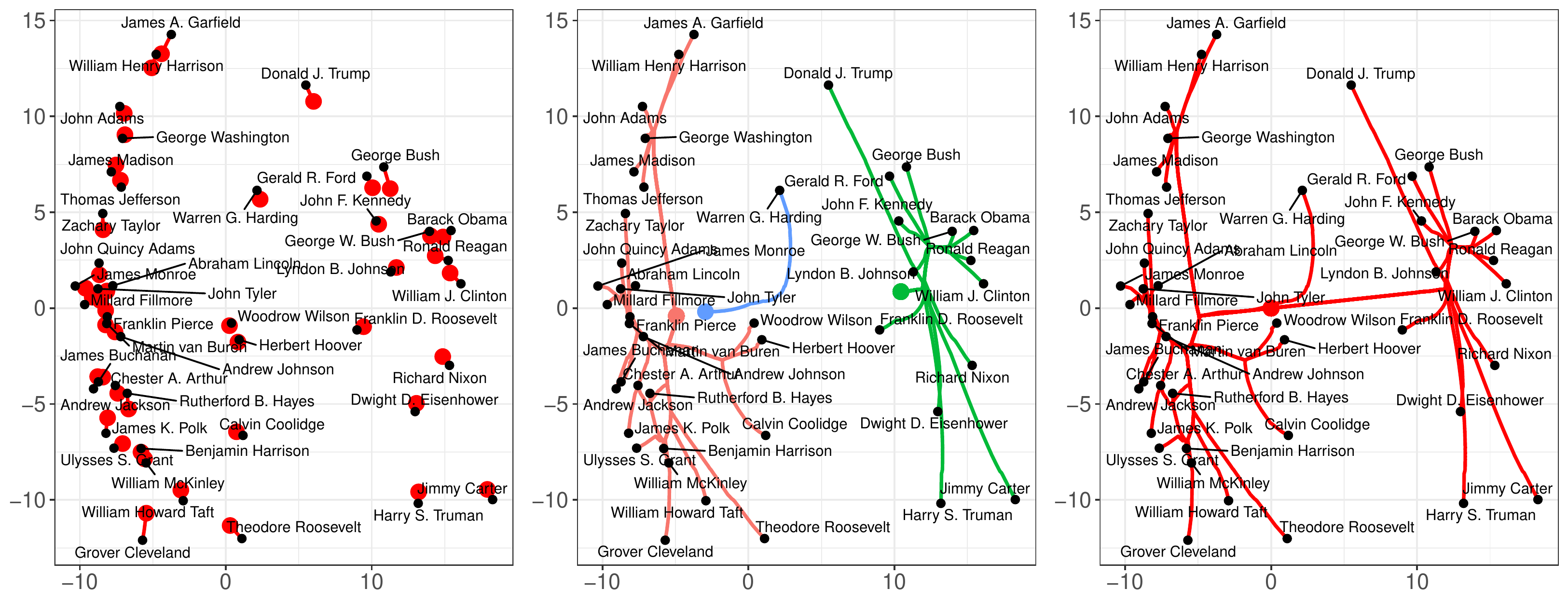}
  \caption{Direct visualization of the solution paths produced by \carp on the \presidents data, corresponding to unclustered (left), partially clustered (middle), and fully clustered (right) solutions. In each panel, the clusters of pre-modern and modern presidents are clearly visible, as is the outlier status of Harding.}
  \label{fig:path_viz}
\end{figure}

Because both the dendrogram and path visualizations are indexed by the regularization level, $\gamma^{(k)}$, it is possible to display them in a ``linked'' fashion, highlighting clusterings on the dendrogram as they are fused in the path plot. Particularly when rendered dynamically, this combination gives the best of both visualizations, combining the global structure visible in the dendrogram with the structural information visible in the path plot. An example of this ``linked'' visualization is shown in Figure \ref{fig:linked_viz}.

\begin{figure}
  \centering
  \includegraphics[width=6in]{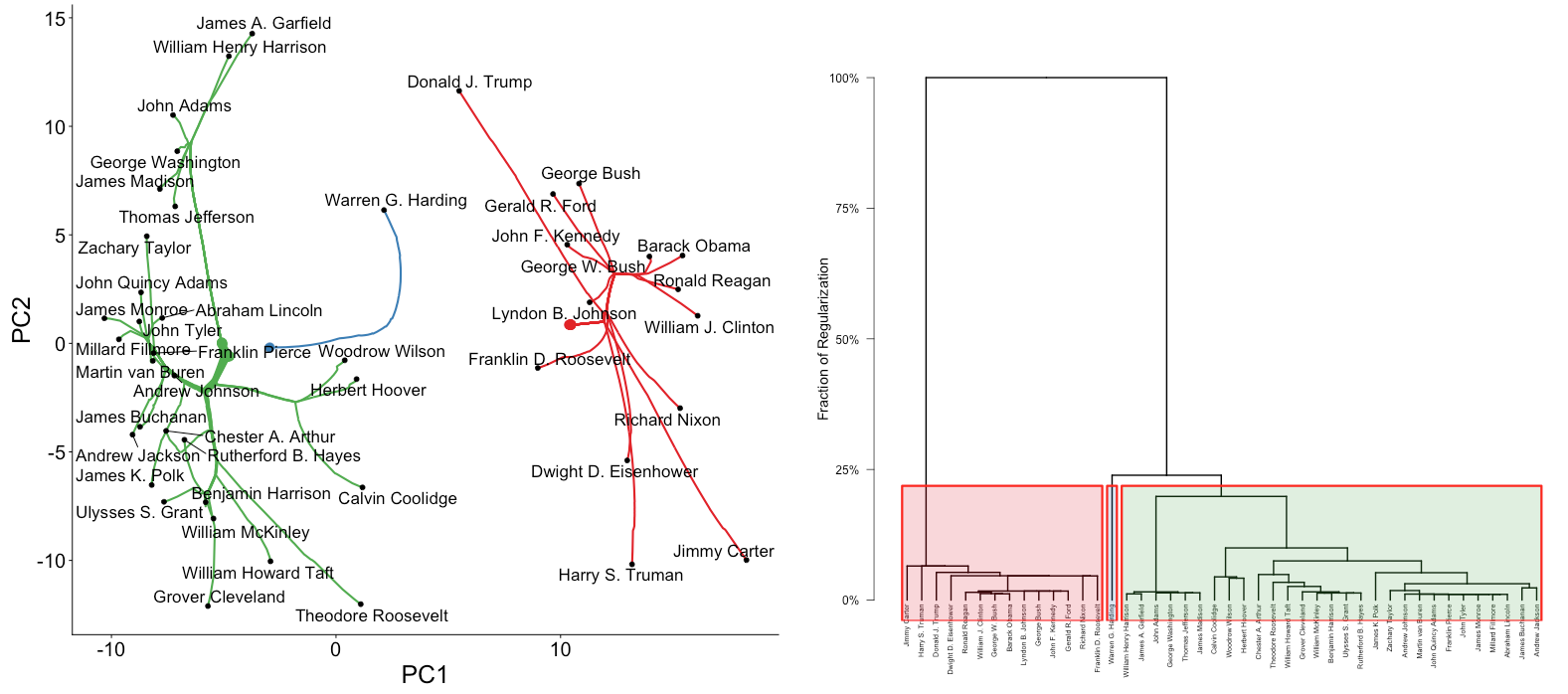}
  \caption{ Linked simultaneous visualization of the \carp dendrogram and path plots. As clusters are formed in the path plot (left), they are highlighted on the dendrogram (right). The clusters of pre-modern and modern presidents are clearly visible, as is the outlier status of Warren G. Harding.}
  \label{fig:linked_viz}
\end{figure}

\section{Discussion} \label{sec:conclusion}
We have introduced Algorithmic Regularization, an iterative one-step approximation scheme which can be used to efficiently obtain high-quality approximations of regularization paths. Algorithmic regularization focuses on accurate reconstruction of an entire regularization path and is particularly useful for obtaining path-wise information, such as a dendrogram or the order in which variables leave the active set in sparse regression. We have focused on the application of the ADMM to convex clustering, but the technique of iterative one step-approximations can be applied to any problem which lacks an efficient algorithm. We believe that algorithmic regularization can be fruitfully applied to a broader range of statistical learning problems and expect that similar computational improvements can be achieved for other difficult optimization problems. 

Theorem \ref{thm:hausdorff} is a novel \emph{global} convergence result, guaranteeing high-quality approximation at each point of the exact solution path. Despite this, there are still many open questions in the analysis of algorithmic regularization. We are particularly interested in determining optimal convergence rates for global path-wise approximation problems and showing that algorithmic regularization can attain those rates. Our proof of Theorem \ref{thm:hausdorff} depends on the strong convexity of the convex clustering problem to ensure linear convergence of the underlying ADMM steps. It would be interesting to explore the interplay between algorithmic regularization and optimization schemes which are linearly convergent without strong convexity, as this may extend the applicability of algorithmic regularization even further. 

%The strong convexity assumption used in the proof of Theorem \ref{thm:hausdorff} appears to be unavoidable, at least for ADMM-based approaches: we have experiments where Hausdorff convergence does not hold for non-strongly-convex problems (\emph{e.g.}, high-dimensional sparse regression).  It would be interesting to explore the interplay between algorithmic regularization and optimization schemes which are linearly convergent without strong convexity, as this may provide an efficient approach to calculating regularization paths for a much wide class of high-dimensional statistical problems.

Using algorithmic regularization, we have introduced the \carp and \cbass algorithms for convex clustering and bi-clustering. On moderately sized problems, \carp and \cbass reduce the time necessary to obtain high-quality regularization paths from several hours to only a few minutes, typically attaining over one-hundred-fold improvements over existing algorithms. Because \carp and \cbass return solutions at a fine grid of the regularization parameter, they can be used to construct accurate convex (bi-)clustering dendrograms, particularly if the back-tracking \carpviz and \cbassviz variants are employed. Additionally, the fine-grained \carp and \cbass solution paths allow for path-wise dynamic visualizations, allowing the analyst to observe exactly how the estimated clusters are formed and structured. 

We anticipate that the computational and visualization techniques proposed in this paper will make convex clustering and bi-clustering an attractive option for applied data analysis. Both \carp and \cbass, as well as the proposed visualizations, are implemented in our \clustRviz software, available at \url{https://github.com/DataSlingers/clustRviz}.

\section*{Acknowledgements}

We thank Eric Chi for helpful discussions about both convex clustering and algorithmic regularization. MW acknowledges support from the NSF Graduate Research Fellowship Program under grant number 1842494. GA acknowledges support from NSF DMS-1554821, NSF NeuroNex-1707400, and NSF DMS-1264058.  JN acknowledges support from NSF DMS-124058, NSF DMS-1554821, and the National Institutes of Health National Cancer Institute T32 Training program in Biostatistics for Cancer Research, Grant Number: CA096520.

MW and JN jointly developed the \clustRviz software. MW is responsible for the content and proof of Theorem \ref{thm:hausdorff} and prepared the manuscript. JN performed initial experiments and developed the back-tracking and post-processing schemes. GA supervised the research and edited the final manuscript.

\section*{Supplementary Materials}

Supplementary materials published alongside the online version of this article contain a more detailed derivation of Algorithms \ref{alg:short_admm} and \ref{alg:short_carp}, a detailed proof of Theorem \ref{thm:hausdorff}, a discussion of our \cbass algorithm for convex bi-clustering, numerical comparisons of \carp to existing clustering approaches on real and synthetic data sets, details about the back-tracking, post-processing, and dendrogram construction strategies used in our \carpviz algorithm and \clustRviz software, and a discussion of additional related work. 

\printbibliography[heading=subbibliography]
\end{refsection}
\clearpage
\appendix

\renewcommand\thefigure{A\arabic{figure}}    
\setcounter{figure}{0}    
\renewcommand\thealgorithm{A\arabic{algorithm}}
\setcounter{algorithm}{0}
\renewcommand\theequation{A\arabic{equation}}
\setcounter{equation}{0}

\begin{center} \Huge \bf Supplementary Materials \end{center}
\begin{refsection}
\section{Operator Splitting Methods for Convex Clustering} \label{app:derivations}
\subsection{ADMM for Convex Clustering}
In this section, we derive and give the full form of the ADMM presented in Algorithm
\ref{alg:short_admm} for the convex clustering problem
\eqref{eqn:cclust}. We begin by noting that, in typical applications,
most of the weights $w_{ij}$ are zero and hence do not enter into the
optimization problem. We can omit the $\binom{n}{2}$-term sum and
instead write the convex clustering problem \eqref{eqn:cclust} as

\begin{equation*}
\argmin_{\bU \in \R^{n \times p}} \frac{1}{2} \|\bX - \bU\|_F^2 + \lambda \left(\sum_{((i, j), w) \in \mathcal{E}} w_{ij}  \|\bU_{i\cdot} - \bU_{j\cdot}\|_q\right)
\end{equation*}

where $\mathcal{E}$ is the set of directed edges with non-zero weights
$w$ connecting $i$ to $j$.

In this form, it is clear that the convex clustering problem is
amenable to operator splitting methods; in particular,
\citet{Chi:2015} showed that the Alternating Direction Method of
Multipliers (ADMM) \citep{Glowinski:1975,Gabay:1976,Boyd:2011} works particularly well for this
problem. Algorithm \ref{alg:admm} differs from the ADMM derived in
\citet{Chi:2015} in two significant ways: firstly, we only consider
edges with non-zero weights, thereby greatly reducing storage
requirements of the algorithm; and secondly, we implement the
algorithm in ``matrix-form'' rather than in a fully vectorized
form. These differences make the resulting algorithm both easier to
derive and to read, as well as more able to take advantage of highly
optimized numerical linear algebra libraries.

We note that while we are solving a matrix-valued problem, it is not a
semi-definite program, and the additional complexity typically
associated with matrix-valued optimization does not apply
here. Because we are optimizing over the space of all matrices of a
certain size, the underlying problem is essentially Euclidean in
geometry and standard (vector-valued) optimization techniques can be
applied, replacing the (squared) Euclidean norm with the (squared)
Frobenius norm and the standard Euclidean inner product with the
Frobenius inner product as necessary.

The derivation of the ADMM for convex clustering \eqref{eqn:cclust} is relatively
straight-forward. We begin by introducing an auxiliary variable $\bV$
containing the pairwise differences between connected rows of
$\bU$. The problem then becomes
\begin{equation*}
\argmin_{\bU \in \R^{n \times p}} \frac{1}{2} \|\bU - \bX\|_F^2 + \lambda \underbrace{\left(\sum_{(e_l, w_l) \in \mathcal{E}} w_l  \|\bV_{l\cdot}\|_q\right)}_{P(\bV; \bw, q)} \quad \text{ subject to } \bD\bU - \bV = 0.
\end{equation*}

From here, we use the scaled form of the ADMM as given by a matrix version of Equations (3.5) to (3.7) of \citet{Boyd:2011}:
\begin{align*}
\bU^{(k+1)} &= \argmin_{\bU \in \R^{n \times p}} \frac{1}{2}\|\bU - \bX\|_F^2 + \frac{\rho}{2}\left\|\bD\bU - \bV^{(k)} + \bZ^{(k)}\right\|_F^2 \\
\bV^{(k+1)} &= \argmin_{\bV \in \R^{|\mathcal{E}| \times p}} \lambda P(\bV; \bw, q) + \frac{\rho}{2}\left\|\bD\bU^{(k + 1)} - \bV + \bZ^{(k)}\right\|_F^2 \\
\bZ^{(k+1)} &= \bZ^{(k)} + \bD\bU^{(k+1)} - \bV^{(k + 1)}
\end{align*}
where the dual variable is denoted by $\bZ$. The analytical solution
to the first subproblem is given by:
\[\bU^{(k+1)} = (\bI + \rho \bD^T\bD)^{-1}\left(\bX + \rho \bD^T(\bV^{(k)} - \bZ^{(k)})\right)\]
This update is the most expensive step in the ADMM, though it can be
significantly sped up by pre-calculating caching the Cholesky
factorization of $\bI + \rho \bD^T\bD$ and using it at each
$\bU$-update:
\[\bU^{(k+1)} = \bL^{-T}\bL^{-1}\left(\bX + \rho \bD^T(\bV^{(k)} - \bZ^{(k)})\right) \quad \text{ where } \quad \bL\bL^T = \bI + \rho \bD^T\bD\]

To solve the second subproblem, we note that it can be written as a
proximal operator: \small
\begin{align*} \argmin_{\bV \in \R^{|\mathcal{E}| \times n}} \lambda P(\bV; \bw, q) + \frac{\rho}{2}\left\|\bD\bU^{(k + 1)} - \bV + \bZ^{(k)}\right\|_F^2 &= \argmin_{\bV \in \R^{|\mathcal{E}| \times n}} \frac{\lambda}{\rho} P(\bV; \bw, q) + \frac{\rho}{2}\left\|\bV - \left(\bD\bU^{(k + 1)} + \bZ^{(k)}\right)\right\|_F^2 \\ &= \prox_{\frac{\lambda}{\rho}P(\cdot; \bw, q)}\left(\bD\bU^{(k + 1)} + \bZ^{(k)}\right)\end{align*} \normalsize
We note that, due to the row-wise structure of $P$, this proximal
operator can be computed separately across the rows of its
argument. In the cases $q = 1$ or $q = 2$, the proximal operator
reduces to element-wise ($q = 1$) or group ($q = 2$) soft-thresholding
row $l$ at the level $w_l \lambda / \rho$. If $q = \infty$, Moreau's
decomposition \citep{Moreau:1962} can be combined with the the
efficient projection onto the $\ell_1$ ball developed by
\citet{Duchi:2008} to evaluate the prox in $\mathcal{O}(p \log p)$
steps. For other values of $q$, an iterative algorithm must be used.

Several stopping criteria for the ADMM have been proposed in the
literature. We have found a simple stopping rule based on the change
in $\bU$ being small sufficient in all cases. While some authors
report speed-ups due to varying the ADMM relaxation parameter $\rho$,
we have found that fixing $\rho$ and re-using the Cholesky factor
$\bL$ to be more efficient. Combining these steps, we obtain Algorithm
\ref{alg:admm}.

\begin{algorithm}[ht]
\caption{Warm-Started ADMM for the Convex Clustering Problem \eqref{eqn:cclust}}
\label{alg:admm}
\begin{enumerate}
\item Input:
\begin{itemize}
\item Data Matrix: $\bX \in \R^{n \times p}$
\item Weighted Directed Edge Set: $\mathcal{E} = \{(e_l, w_l)\}$
\item Relaxation Parameter: $ \rho \in \R_{> 0}$
\item Initial Regularization Parameter $\epsilon$ and Multiplicative Step-Size $t$
\end{itemize}
\item Precompute:
\begin{itemize}
\item Difference Matrix: $\bD \in \R^{|\mathcal{E}| \times n}$ where $D_{ij}$ is $1$ if  edge $i$ begins at node $j$, $-1$ if edge $i$ ends at node $j$, and $0$ otherwise
\item Cholesky Factor: $\bL = \textsf{chol}(\bI + \rho \bD^T\bD) \in \R^{n \times n}$
\end{itemize}
\item Initialize:
\begin{itemize}
\item $\bU^{(0)} = \bX$
\item $\bV^{(0)} = \bZ^{(0)}  = \bD\bX$
\item $l = 0$, $\lambda_0 = \epsilon$, $k = 0$, 
\end{itemize}
\item Repeat until $\|\bV^{(k)}\| = 0$:
\begin{itemize}
\item Repeat until convergence:
\begin{enumerate}
\item[(i)] $\bU^{(k + 1)} = \bL^{-T}\bL^{-1}\left(\bX + \rho \bD^T(\bV^{(k)} - \bZ^{(k)}\right)$
\item[(ii)] $\bV^{(k+1)} = \prox_{\lambda_l / \rho\, P(\cdot; \bw, q)}\left(\bD \bU^{(k + 1)} + \bZ^{(k)}\right)$
\item[(iii)] $\bZ^{(k + 1)} = \bZ^{(k)} + \bD \bU^{(k + 1)} - \bV^{(k+1)}$
\item[(iv)] $k := k + 1$
\end{enumerate}
\item Store $\hat{\bU}_{\lambda_l} = \bU^{(k)}$
\item Update Regularization Parameter $l := l + 1$; $\lambda_l := \lambda_{l - 1} * t$
\end{itemize}
\item Return $\left\{\hat{\bU}_{\lambda_i}\right\}_{i = 0}^{l - 1}$ as the regularization path
\end{enumerate}
\end{algorithm}

\subsection{Algorithmic Regularization for Convex Clustering}
In this section, we give a the full version of the \carp algorithm presented in Algorithm \ref{alg:short_carp}. 
\carp can be obtained from the standard ADMM for convex clustering (Algorithm \ref{alg:admm}) by replacing the inner ADMM loop with a single iteration. This modification gives \carp (Algorithm \ref{alg:carp}). As with Algorithm \ref{alg:admm}, we prefer to use a matrix formulation, instead of a fully vectorized formulation, to simplify the implementation and to take advantage of high-performance numerical linear algebra libraries.

\begin{algorithm}[ht]
\caption{\carp: Convex Clustering via Algorithmic Regularization Paths}
\label{alg:carp}
\begin{enumerate}
\item Input:
\begin{itemize}
\item Data Matrix: $\bX \in \R^{n \times p}$
\item Weighted Directed Edge Set: $\mathcal{E} = \{(e_l, w_l)\}$
\item Relaxation Parameter: $ \rho \in \R_{> 0}$
\item Initial Regularization Parameter $\epsilon$ and Multiplicative Step-Size $t$
\end{itemize}
\item Precompute:
\begin{itemize}
\item Difference Matrix: $\bD \in \R^{|\mathcal{E}| \times n}$ where $D_{ij}$ is $1$ if  edge $i$ begins at node $j$, $-1$ if edge $i$ ends at node $j$, and $0$ otherwise
\item Cholesky Factor: $\bL = \textsf{chol}(\bI + \rho \bD^T\bD) \in \R^{n \times n}$
\end{itemize}
\item Initialize:
\begin{itemize}
\item $\bU^{(0)} = \bX$
\item $\bV^{(0)} = \bZ^{(0)}  = \bD\bX$
\item $k = 0$, $\gamma^{(0)} = \epsilon$ 
\end{itemize}
\item Repeat until $\|\bV^{(k)}\| = 0$:
\begin{enumerate}
\item[(i)] $\bU^{(k + 1)} = \bL^{-T}\bL^{-1}\left(\bX + \rho \bD^T(\bV^{(k)} - \bZ^{(k)}\right)$
\item[(ii)] $\bV^{(k+1)} = \prox_{\gamma^{(k)} / \rho\, P(\cdot; \bw, q)}\left(\bD \bU^{(k + 1)} + \bZ^{(k)}\right)$
\item[(iii)] $\bZ^{(k + 1)} = \bZ^{(k)} + \bD \bU^{(k + 1)} - \bV^{(k+1)}$
\item[(iv)] $k := k + 1$, $\gamma^{(k)} = \gamma^{(k - 1)} * t$
\end{enumerate}
\item Return $\left\{\bU^{(k)}\right\}_{i = 0}^k$ as the \carp algorithmic regularization path
\end{enumerate}
\end{algorithm}

\clearpage

\subsection{AMA for Convex Clustering}
In addition to the AMA, \citet{Chi:2015} also show that the convex clustering problem \eqref{eqn:cclust} can be efficiently solved using the Alternating Minimization Algorithm (AMA) of \citet{Tseng:1991}. In our notation, the AMA becomes
\begin{align*}
\bU^{(k+1)} &= \argmin_{\bU \in \R^{n \times p}} \frac{1}{2}\|\bU - \bX\|_F^2 + \langle \bZ^{(k)}, \bD\bU - \bV^{(k)} \rangle\\
\bV^{(k+1)} &= \argmin_{\bV \in \R^{|\mathcal{E}| \times n}} \lambda P(\bV; \bw, q) + \langle \bZ^{(k)}, \bD\bU^{(k+1)} - \bV^{(k)}\rangle + \frac{\rho}{2}\left\|\bD\bU^{(k+1)} - \bV\right\|_F^2 \\
\bZ^{(k+1)} &= \bZ^{(k)} + \rho(\bD\bU^{(k+1)} - \bV^{(k + 1)})
\end{align*}
(Note that we use the unscaled updates for $\bV, \bZ$ here as the AMA uses different values of the relaxation parameter in the $\bU$ and $\bV$ updates. In particular, this means that the dual variables $\bZ$ from the ADMM are not the same as those from the AMA.) Simplifying these updates as before, the AMA becomes:
\begin{align*}
\bU^{(k+1)} &= \bX - \bD^T\bZ^{(k)}\\
\bV^{(k+1)} &=  \prox_{\frac{\lambda}{\rho}P(\cdot; \bw, q)}\left(\bD\bU^{(k + 1)} + \bZ^{(k)}/\rho\right)\\
\bZ^{(k+1)} &= \bZ^{(k)} + \rho(\bD\bU^{(k+1)} - \bV^{(k + 1)})
\end{align*}
\citet{Chi:2015} note that a clever application of Moreau's decomposition \citep{Moreau:1962} allows the $\bV$-updates to be elided and for the AMA to be simplified into a two-step scheme. The $\bV^{(k)}$ iterates are key to dendrogram reconstruction, however, so such a simplification could not be used in an AMA-based \carp variant. 

In our experiments, this elision is necessary for the AMA to outperform the ADMM and so, without it, the AMA does not appear to be a promising basis for an algorithmic regularization scheme. In general, the ADMM appears to converge more rapidly \emph{per iteration} than the AMA, while the simplified AMA has much faster updates, allowing better overall computational performance in a standard optimization scheme. Since \carp performs only a single iteration per regularization level, however, the faster per iteration convergence of the ADMM is more important to us than the faster calculation of the AMA.

Finally, \citet{Chi:2015} also discuss the use of accelerated variants of the ADMM and AMA \citep{Goldstein:2014} to improve convergence. Because \carp uses only a single iteration for each regularization level, it is not amenable to acceleration. 

%\clearpage

\section{Proof of Theorem \ref{thm:hausdorff}}
\label{app:proof}

In this section we prove Theorem \ref{thm:hausdorff} on the Hausdorff convergence of \carp to the convex clustering regularization path. We begin with 3 technical lemmas which may be of independent interest: Lemma \ref{lem:q_linear} provides a convergence rate for the optimization step embedded within a \carp iteration; Lemma \ref{lem:lip_paths} establishes a form of Lipschitz continuity for convex clustering regularization paths; Lemma \ref{lem:err_bound} provides a global bound for the approximation error induced by \carp at any iteration. In one step, our results are stated and proven for \carp with an $\ell_2$-fusion penalty, but can be easily extended to other $\ell_q$-fusion penalties.

\vspace{0.05in}
\begin{lem}[Q-Linear Error Decrease] \label{lem:q_linear}
At each iteration $k$, the \carp approximation error decreases by a
factor $c < 1$ not depending on $t$ or $\epsilon$. That is,
\[\|\bU^{(k)} - \hat{\bU}_{\gamma^{(k)}}\| + \|\bZ^{(k)} - \hat{\bZ}_{\gamma^{(k)}}\| < c \left[\|\bU^{(k-1)} - \hat{\bU}_{\gamma^{(k)}}\| + \|\bZ^{(k-1)} - \hat{\bZ}_{\gamma^{(k)}}\|\right]\]
for some $c$ strictly less than 1.
\end{lem}

\begin{proof}
  By construction, each \carp step is a single iteration of
  the ADMM for the convex clustering problem (Algorithm \ref{alg:admm}) initialized at $(\bU^{(k-1)}, \bV^{(k-1)}, \bZ^{(k-1)})$. Hence it suffices to analyze the convergence of the ADMM for the convex clustering problem and to establish linear convergence.
  
  The convex clustering problem \eqref{eqn:cclust} is strongly convex due to squared Frobenius norm term. Linear convergence of the standard ADMM for strongly convex problems was first shown by \citet{Lions:1979} and has since been refined by several other authors including \citet{Shi:2014}, \citet{Nishihara:2015}, \citet{Deng:2016}, and \citet{Yang:2016}.
  
  In vectorized form, with $\bx = \vecop(\bX)$, $\bu = \vecop(\bU)$, and $\bv = \vecop(\bV)$, the convex clustering problem \eqref{eqn:cclust} can be expressed as:
  \[\argmin_{\bu \in \R^{np}, \bv \in \R^{|\mathcal{E}|p}} \frac{\|\bx - \bu\|_2^2}{2} + \lambda \|\bv\|_{\vecop(\ell_q)} \quad \text{ subject to } \quad (\bI \otimes \bD)\bu = \bv\]
  where $\|\cdot\|_{\vecop(\ell_q)}$ is an appropriately vectorized version of the row-wise $\ell_q$ norm (a standard $\ell_1$ norm in the case $q = 1$ and a mixed $\ell_q/\ell_1$ norm otherwise) and we have omitted the fusion weights for brevity.
  
  In the notation of \citet{Hong:2017}, the constraint matrix for the convex clustering problem is given by $\bE = \begin{pmatrix} \bI \otimes \bD & -\bI \end{pmatrix}$, for appropriately sized identity matrices, which is clearly row-independent, yielding linear convergence of the primal and dual variables at a rate $c_{\lambda} < 1$ which may depend on $\lambda$. (We do not need to verify their additional technical assumptions as we are only using a two-block ADMM instead of the more general multi-block ADMM which is the focus of their paper.) Taking $c = \sup_{\lambda \leq \lambda_{\max}} c_{\lambda}$, we observe that the \carp iterates are uniformly Q-linearly convergent at a rate $c$.
\end{proof}

\begin{rem}
  Recently, \citet{Deng:2016} gave a readable and precise analysis of the linear convergence of the ADMM, including estimates of the convergence rate $c$. The specific proof technique they employ does not strictly apply to the convex clustering problem \ref{eqn:cclust}, however, as the $\bD$ matrix is rank-deficient (excluding their Scenario 1) and the norm used for the fusion penalty is non-differentiable at the origin (excluding their Scenario 3). If an estimate of the convergence rate is required, the analysis of \citet{Deng:2016} can be applied to the convex clustering problem \ref{eqn:cclust} by re-parameterizing the problem to address the rank-deficiency of $\bD$. In particular, if the redundant rows of $\bD$ are combined (eliminating the nullspace of $\bD$), the resulting matrix will be full-row rank, allowing Scenario 1 and Case 2 of \citet{Deng:2016} to be applied. This re-parameterization results in different split and dual variables ($\bV$ and $\bZ$, corresponding to the $\bD$ matrix), however, so we do not pursue that approach here. The primal variable, $\bU$, remains unchanged under this re-parameterization.
\end{rem}

\vspace{0.1in}

\begin{lem}[Lipschitz Continuity of Solution Paths] \label{lem:lip_paths}

  $(\hat{\bU}_{\lambda}, \hat{\bZ}_{\lambda})$ is $L$-Lipschitz with
  respect to $\lambda$. That is,
  \[\|\hat{\bU}_{\lambda_1} - \hat{\bU}_{\lambda_2}\| +  \|\hat{\bZ}_{\lambda_1} - \hat{\bZ}_{\lambda_2}\| \leq L*|\lambda_1 - \lambda_2|\]
  for some $L > 0$.

\end{lem}

We note that this not the only form of Lipschitz continuity commonly
considered for regularized estimation problems. In particular,
Lipschitz continuity of the \emph{solution with respect to the data}
is a key element of various consistency results, while Lipschitz
continuity of the \emph{objective function with respect to the
  parameters} is a key assumption used to prove convergence of many
optimization schemes.

\begin{proof} It suffices to prove Lipschitz continuity of $\hat{\bU}_{\lambda}$ and
$\hat{\bZ}_{\lambda}$ separately and then take the sum of their
  Lipschitz moduli as the joint Lipschitz modulus.

We first show that $\hat{\bU}_{\lambda}$ is Lipschitz. In vectorized
form, the convex clustering problem is
\[\hat{\bu}_{\lambda} = \argmin_{\bu \in \R^{np}} \frac{1}{2}\|\bu - \bx\|_2^2 + \lambda f_q(\tilde{\bD}\bu)\]
where $\bu = \vecop(\bU)$, $\bx = \vecop(\bX)$, $f_q$ is a convex function, and $\tilde{\bD} = \bI \otimes \bD$
is a fixed matrix \citep[\emph{cf.}][]{Tan:2015}.

The KKT conditions give
\[0 \in \bu_{\lambda} - \bx + \lambda \tilde{\bD}^T\partial f_q(\tilde{\bD}\bu_{\lambda}) \]
where $\partial f_q(\cdot)$ is the subdifferential of $f_q$. Since $f_q$ is
convex, it is differentiable almost everywhere \citep[Theorem 25.5]{Rockafellar:1970},
so the following holds for almost all $\bu_{\lambda}$:
\[0 = \bu_{\lambda} - \bx + \lambda \tilde{\bD}^T f_q'(\tilde{\bD}\bu_{\lambda}) \]
Differentiating with respect to $\lambda$, we obtain \citep[\emph{c.f.}][]{Rosset:2007}
\begin{align*}
  0 &= \bu_{\lambda} - \bx + \lambda \tilde{\bD}^Tf'_q(\tilde{\bD}\bu_{\lambda}) \\
  \frac{\partial}{\partial \lambda}\left[0\right] &= \frac{\partial}{\partial \lambda}\left[\bu_{\lambda} - \bx + \lambda \tilde{\bD}^Tf'_q(\tilde{\bD}\bu_{\lambda})\right] \\
  0 &= \frac{\partial \bu_{\lambda}}{\partial \lambda} - 0 + \lambda \frac{\partial}{\partial \lambda}\left[\tilde{\bD}^Tf'_q(\tilde{\bD}\bu_{\lambda})\right] + \tilde{\bD}^Tf'_q(\tilde{\bD}\bu_{\lambda}) \\
  0 &= \frac{\partial \bu_{\lambda}}{\partial \lambda} + \lambda \tilde{\bD}^Tf_q''(\tilde{\bD}\bu_{\lambda})\tilde{\bD}\frac{\partial \bu_{\lambda}}{\partial \lambda} + \tilde{\bD}^Tf'_q(\tilde{\bD}\bu_{\lambda}) \\
\implies \frac{\partial \bu}{\partial \lambda} &= -[\bI + \lambda\tilde{\bD}^T f''_q(\tilde{\bD}\bu) \tilde{\bD}]^{-1} \bD^Tf'_q(\tilde{\bD}\bu).
\end{align*}
Note that $\bu_{\lambda}$ depends on $\lambda$ so the chain rule must be used here. From here, we note
\[\left\|\frac{\partial \bu_{\lambda}}{\partial \lambda}\right\|_{\infty}=\left\| -[\bI + \lambda\tilde{\bD}^T f''_q(\tilde{\bD}\bu_{\lambda}) \tilde{\bD}]^{-1} \tilde{\bD}^Tf'_q(\tilde{\bD}\bu_{\lambda}) \right\|_{\infty}\leq\| -[\bI + 0]^{-1} \tilde{\bD}^Tf'_q(\tilde{\bD}\bu_{\lambda}) \|_{\infty} = \|\tilde{\bD}^T f'_q(\tilde{\bD}\bu_{\lambda})\|_{\infty}.\]

For the convex clustering problem, we recall that $f_q(\cdot)$ is a norm and hence has bounded gradient; hence $f'_q(\tilde{\bD}\bu_{\lambda})$ is bounded so the gradient of the regularization path is bounded and exists almost everywhere. This implies that the regularization path is \emph{piecewise} Lipschitz. Since the solution path is constant for $\lambda \geq \lambda_{\max}$ and is continuous \citep[Proposition 2.1]{Chi:2015}, the solution path is globally Lipschitz with a Lipschitz modulus equal to the maximum of the piecewise Lipschitz moduli.

A similar argument shows Lipschitz continuity of $\hat{\bZ}_{\lambda}$ or one can use the relationships between $\hat{\bU}_{\lambda}$ and $\hat{\bZ}_{\lambda}$ discussed
in Section 2.1 of \citet{Tan:2015}.

\end{proof}

\begin{lem}[Global Error Bound] \label{lem:err_bound}
The following error bound holds for all $k$:

\[\|\bU^{(k)} - \hat{\bU}_{\gamma^{(k)}}\| + \|\bZ^{(k)} - \hat{\bZ}_{\gamma^{(k)}} \|\leq c^kL\epsilon + L(t - 1)\epsilon t^k \sum_{i=1}^{k-1} \left(\frac{c}{t}\right)^i\]
\end{lem}

\begin{proof} Throughout, we let \[\hat{\bW}_{\lambda} = \begin{pmatrix} \hat{\bU}_{\lambda} \\ \hat{\bZ}_{\lambda}\end{pmatrix} \text{ and } \hat{\bW}^{(k)} = \begin{pmatrix} \bU^{(k)} \\ \bZ^{(k)} \end{pmatrix}.\]

Our proof proceeds by induction on $k$. First note that, at initialization:
\[\|\bW^{(0)} - \hat{\bW}_{\epsilon}\| \leq L\epsilon\]
by Lemma \ref{lem:lip_paths}.

Next, at $k=1$, we note that
\[\|\bW^{(1)} - \hat{\bW}_{t\epsilon}\| \leq c \|\bW^{(0)} - \hat{\bW}_{t\epsilon}\| \]
by Lemma \ref{lem:q_linear}. We now the triangle inequality to split the right hand side:
\[\|\bW^{(0)} - \hat{\bW}_{t\epsilon}\| \leq \underbrace{\|\bW^{(0)} - \hat{\bW}_{\epsilon}\|}_{\text{RHS-1}} + \underbrace{\|\hat{\bW}_{\epsilon} - \hat{\bW}_{t\epsilon}\|}_{\text{RHS-2}}\]
From above, we have $\text{RHS-1} \leq L\epsilon$. Using Lemma \ref{lem:lip_paths},
$\text{RHS-2}$ can be bounded by
\[\|\hat{\bW}_{\epsilon} - \hat{\bW}_{t\epsilon}\| \leq L \left|t\epsilon - \epsilon\right| = L(t-1)\epsilon.\]

Putting these together, we get
\[\|\bW^{(1)} - \hat{\bW}_{t\epsilon}\| \leq c \left[\text{RHS-1} + \text{RHS-2}\right] \leq c\left[L\epsilon + L(t-1)\epsilon\right] = cLt\epsilon \]

Repeating this argument for $k=2$, we see
\begin{align*}
\|\bW^{(2)} - \hat{\bW}_{t^2\epsilon}\| &\leq c \|\bW^{(1)} - \hat{\bW}_{t^2\epsilon}\| \\
                                        &\leq c\left[\|\bW^{(1)} - \hat{\bW}_{t\epsilon}\| + \|\hat{\bW}_{t\epsilon} - \hat{\bW}_{t^2\epsilon}\|\right] \\
                                        &\leq c\left[cLt\epsilon + L\left|t^2\epsilon - t\epsilon\right|\right] \\
                                        &= c^2Lt\epsilon + cL(t-1)\epsilon * t \\
                                        &= c^2Lt\epsilon + L\epsilon(t-1)t^2 * \left(\frac{c}{t}\right) \\
                                        &= c^2Lt\epsilon + L\epsilon(t-1)t^2 * \sum_{i=1}^{k-1} \left(\frac{c}{t}\right)^i
                                        \end{align*}

We use this as a base case for our inductive proof and prove the  general case:
\begin{align*}
\|\bW^{(k)} - \hat{\bW}_{t^k\epsilon}\| &\leq c \|\bW^{(k-1)} - \hat{\bW}_{t^k\epsilon}\| \\
                                        &\leq c\left[\|\bW^{(k-1)} - \hat{\bW}_{t^{k-1}\epsilon}\| + \|\hat{\bW}_{t^{k-1}\epsilon} - \hat{\bW}_{t^k\epsilon}\|\right] \\
                                        &\leq c\left[c^{k-1}Lt\epsilon + L\epsilon(t-1)t^{k-1} \sum_{i=1}^{k-2} \left(\frac{c}{t}\right)^i + L\left|t^k\epsilon - t^{k-1}\epsilon\right|\right] \\
                                        &= c^{k}Lt\epsilon + cL\epsilon(t-1)t^{k-1}\sum_{i=1}^{k-2}\left(\frac{c}{t}\right)^i + cL\epsilon(t^k - t^{k-1}) \\
                                        &= c^{k}Lt\epsilon + L\epsilon(t-1)t^k\left[\frac{c}{t}\sum_{i=1}^{k-2}\left(\frac{c}{t}\right)^i + \frac{c}{t}\right] \\
                                        &= c^{k}Lt\epsilon + L\epsilon(t-1)t^k\left[\sum_{i=2}^{k-1}\left(\frac{c}{t}\right)^i + \frac{c}{t}\right] \\
                                        &= c^{k}Lt\epsilon + L\epsilon(t-1)t^k\sum_{i=1}^{k-1}\left(\frac{c}{t}\right)^i
                                        \end{align*}
Expanding the definitions of $\bW^{(k)}, \hat{\bW}_{\lambda}$, we get the
desired result.

\end{proof}

With these results, we are now ready to prove Theorem \ref{thm:hausdorff}:

\hausdorff*

\begin{proof}[Proof of Theorem \ref{thm:hausdorff}]

It suffices to show that $\{\bW^{(k)}\}, \{\hat{\bW}_{\lambda}\}$
converge in the Hausdorff metric to show that the primal and dual paths converge
separately. We break our proof into three steps:
\begin{enumerate}[i.]
\item $\sup_{\lambda} \inf_k \left\|\bW^{(k)} - \hat{\bW}_{\lambda}\right\| \to 0$;
\item $\epsilon t^{k^*}$ remains bounded as $t, \epsilon$ decrease and $k^*$ increases, where $k^*$ is the iteration at which \carp terminates; and
\item $\sup_{k} \inf_{\lambda} \left\|\bW^{(k)} - \hat{\bW}_{\lambda}\right\| \to 0$.
\end{enumerate}
Together, these give the desired result.

\textbf{Step i.} We first show that \[\sup_{\lambda} \inf_k \left\|\bW^{(k)} - \hat{\bW}_{\lambda}\right\|\]
tends to zero. We begin by fixing temporarily $\lambda$ and bounding
\[\inf_k \left\|\bW^{(k)} - \hat{\bW}_{\lambda}\right\|\]
The infimum over all $k$ is less than the distance at any particular $k$,
so it suffices to choose \emph{a} value of $k$ which gives convergence to 0. Let $\tilde{k}$ be the value of $k$ which
gives the closest value of $\gamma^{(k)}$ to $\lambda$ along the \carp path; and let
$\tilde{\lambda} = \gamma^{(\tilde{k})} = \epsilon t^{\tilde{k}}$. That is,
\[\tilde{k} = \argmin_k |\gamma^{(k)} - \lambda| \quad \text{ and } \tilde{\lambda} = \gamma^{(\tilde{k})}\]
Then
\[\inf_k \left\|\bW^{(k)} - \hat{\bW}_{\lambda}\right\| \leq \|\bW^{(\tilde{k})} - \hat{\bW}_{\lambda}\| \leq \underbrace{\|\bW^{(\tilde{k})} - \hat{\bW}_{\tilde{\lambda}}\|}_{\text{RHS-1}} + \underbrace{\|\hat{\bW}_{\tilde{\lambda}} - \hat{\bW}_{\lambda}\|}_{\text{RHS-2}}\]

Using Lemma \ref{lem:lip_paths}, we can bound $\text{RHS-2}$ as
\[\text{RHS-2} \leq L |\tilde{\lambda} - \lambda| \leq L | \gamma^{(\tilde{k} + 1)} - \gamma^{(\tilde{k} - 1)}| = L * \epsilon t^{\tilde{k}-1} * [t^2-1] = L * \epsilon t^{\tilde{k}-1} * [t^2 -1] \leq L * \lambda_{\max} * [t^2 -1]\]

Using Lemma \ref{lem:err_bound}, we can bound $\text{RHS-1}$ as
\begin{align}
\text{RHS-1} &= \|\bW^{(\tilde{k})} - \hat{\bW}_{\tilde{\lambda}}\| \notag \\ & = \|\bW^{(\tilde{k})} - \hat{\bW}_{\gamma^{(\tilde{k})}}\| \notag \\ &= c^{\tilde{k}}L\epsilon + L(t-1) * \epsilon t^{\tilde{k}} \sum_{i=1}^{k-1} \left(\frac{c}{1+t}\right)^i \leq c^{\tilde{k}}L\epsilon + L(t-1) * \epsilon t^{\tilde{k}} * C \label{eqn:rhs_bound}
\end{align} where $C = \sum_{i=1}^{\infty} \left(\frac{c}{1+t}\right)^i$ is large but finite.
Since $c < 1$ and $\tilde{\lambda} = \epsilon t^{\tilde{k}} \leq \lambda_{\max}$, we can replace the $k$-dependent
quantities to get
\[\text{RHS-1} = \|\bW^{(\tilde{k})} - \hat{\bW}^{\tilde{\lambda}}\| \leq L\epsilon + C * L(t-1) * \lambda_{\max}\]

Putting these together, we have
\[\inf_{k} \|\bW^{(k)} - \hat{\bW}_{\lambda}\| \leq \text{RHS-1} + \text{RHS-2} \leq L\epsilon + C * L(t-1) * \lambda_{\max} + L * \lambda_{\max} * [t^2 -1]\]
Since the right-hand side doesnt' depend on $\lambda$, we have
\[\sup_{\lambda} \inf_{k} \|\bW^{(k)} - \hat{\bW}_{\lambda}\| \leq L\epsilon + C * L(t-1) * \lambda_{\max} + L * \lambda_{\max} * [t^2 -1]\]
As $(t, \epsilon) \to (1, 0)$, we have that the right-hand side converges to zero
and hence
\[\sup_{\lambda} \inf_{k} \|\bW^{(k)} - \hat{\bW}_{\lambda}\| \to 0\]
as desired.

\textbf{Step ii.} Before showing the other half of Hausdorff convergence, we pause
to prove an intermediate result: As $(t, \epsilon) \to (1, 0)$, $\epsilon t^{k^*}$ remains bounded,
where $k^* = k^*(t, \epsilon)$ is the iteration at which \carp halts.
For this step, we specialize to the $\ell_2$-case for concreteness, though our results
are easily generalized.

\carp terminates when $\|\bV^{(k+1)}\|_{\infty, q} = \max_{i, j} \|\bU^{(k+1)}_{i\cdot} - \bU^{(k+1)}_{j\cdot}\|_q = 0$;
that is, \carp terminates when all of the pairwise differences have gone to zero and the data has
been fused into a single cluster.

Note that the update
\[\bV^{(k + 1)}_{i} = \left[1 - \frac{w_i \gamma^{(k)}}{\|(\bD\bU^{(k+1)} + \bZ^{(k)})_i\|_2}\right](\bD\bU^{(k+1)} + \bZ^{(k)})_i\]
will set $\bV^{(k + 1)}_{i}$ to zero when
\[\|(\bD\bU^{(k+1)} + \bZ^{(k)})_i\|_2 < w_{i} \gamma^{(k)}\]
Letting $(j, k)$ be the endpoints of edge $i$, we find
\begin{align*}
\|(\bD\bU^{(k+1)} + \bZ^{(k)})_i\|_2 &= \|\bU^{(k+1)}_{j\cdot} - \bU^{(k+1)}_{k\cdot} + \bZ^{(k)}_{i\cdot}\|_2 \\
                                     &= \left\|(\bU^{(k+1)}_{j\cdot} - \overline{\bx}) - (\bU^{(k+1)}_{k\cdot} -\overline{\bx}) + \bZ^{(k)}_{i\cdot}\right\|_2 \\
                                     &\leq \left\|\bU^{(k+1)}_{j\cdot} - \overline{\bx}\right\|_2 + \left\|\bU^{(k+1)}_{k\cdot} -\overline{\bx}\right\|_2 + \left\| \bZ^{(k)}_{i\cdot}\right\|_2
\end{align*}
where $\overline{\bx}$ is the column-wise mean of $\bX$.

Our strategy will be to show that this quantity is less that $w_{i} \gamma^{(k)}$
for some $k > k^{*}$ small enough that $\epsilon t^{k}$ remains bounded
and hence $\epsilon t^{k^*}$ remains bounded. Let \[\tilde{k} = \left\lceil\frac{\log(\lambda_{\max}/\epsilon)}{\log(t)}\right\rceil = \lceil\log_{t}(\lambda_{\max}/\epsilon)\rceil\]
be the first value of $k$ such that $\gamma^{(k)} > \lambda_{\max}$, \emph{i.e.}, the
value of $\lambda$ such that all of the pairwise differences have gone to zero
and the data has been fused into a single cluster in the regularization path ($\tilde{k}$ is to the regularization path as $k^*$ is to the \carp path).

Using the bound from Equation \eqref{eqn:rhs_bound}, we have that
\[\|\bU^{(k+1)}_{j\cdot} - \overline{\bx}\| = \|\bu^{(k+1)}_l - (\hat{\bU}_{\lambda_{\max}})_{j\cdot}\| < L\epsilon + C * L(t-1) * \lambda_{\max}\]
so
\[\|\bU^{(k+1)}_{j\cdot}- \overline{\bx}\| + \|\bU^{(k+1)}_{k\cdot} - \overline{\bx}\| < 2\left(L\epsilon + C * L(t-1) * \lambda_{\max}\right)\]

Bounding $\|\bZ^{(k)}_{i}\|_2$ is more subtle, but a rough bound can be obtained
again using Equation \eqref{eqn:rhs_bound} to obtain:
\[\|\bZ^{(k)}_{i\cdot} - (\hat{\bZ}_{\lambda_{\max}})_{i\cdot}\|_2 < L\epsilon + C * L(t-1) * \lambda_{\max}\]
so
\[\|\bZ^{(k)}_{i\cdot}\| \leq  \|(\hat{\bZ}_{\lambda_{\max}})_{i\cdot}\|_2 + L\epsilon + C * L(t-1) * \lambda_{\max}\]

Putting these together, we obtain
\[\|\bu^{(k+1)}_l - \bu^{(k+1)}_m - \bz^{(k)}_{l, m}\|_2 \leq \|(\hat{\bz}_{\lambda_{\max}})_{l, m}\|_2 + 3\left(L\epsilon + C * L(t-1) * \lambda_{\max}\right)\]

To stop, we require that
\[\|(\hat{\bZ}_{\lambda_{\max}})_{i\cdot}\|_2 + 3\left(L\epsilon + C * L(t-1) * \lambda_{\max}\right) < w_{l, m} \underbrace{\epsilon t^k}_{\gamma^{(k)}}\]
which occurs when
\[k > \log_{t}\frac{\|(\hat{\bZ}_{\lambda_{\max}})_{i\cdot}\|_2 + 3(L\epsilon + C * L(t-1) * \lambda_{\max})}{w_{l, m} \epsilon}\]
Taking the max over all $(l, m)$-pairs we find
\[k^* \leq \max_{l, m} \log_{t}\frac{\|(\hat{\bz}_{\lambda_{\max}})_{l, m}\|_2 + 3(L\epsilon + C * L(t-1) * \lambda_{\max})}{w_{l, m} \epsilon}\]
Hence it suffices to note
\[\epsilon t^{k^*} \leq \epsilon t^{\max_{i} \log_{t}\frac{\|(\hat{\bZ}_{\lambda_{\max}})_{i\cdot}\|_2 + 3(L\epsilon + C * L(t-1) * \lambda_{\max})}{w_{i} \epsilon}} \leq \max_{i}  \frac{\|(\hat{\bZ}_{\lambda_{\max}})_{i\cdot}\|_2 + 3(L\epsilon + C * L(t-1) * \lambda_{\max})}{w_{i}}\]
which clearly remains bounded as $t, \epsilon \to (1, 0)$.

\textbf{Step iii.} With this result in hand, the proof is similar to the first half.
Again, we can invoke Lemma \ref{lem:err_bound} to find that
\[\inf_{\lambda} \|\bW^{(k)} - \hat{\bW}_{\lambda}\| \leq \|\bW^{(k)} - \hat{\bW}_{\epsilon t^k}\| \leq c^k L\epsilon + CL * (t-1) * \epsilon t^k \]
With the result from above, $\epsilon t^k$ remains bounded above by some $B < \infty$,
so
\[\sup_k \inf_{\lambda} \|\bW^{(k)} - \hat{\bW}_{\lambda}\| = \sup_{1 \leq k \leq k^*} \inf_{\lambda} \|\bW^{(k)} - \hat{\bW}_{\lambda}\| \leq \sup_{1 \leq k \leq k^*} c^k L\epsilon + CL * (t-1) * \epsilon t^k \leq L\epsilon + CL * (t-1) * B\]

As $(t, \epsilon) \to (1, 0)$, the right hand-side goes to zero so
\[\sup_k \inf_{\lambda} \|\bW^{(k)} - \hat{\bW}_{\lambda}\| \to 0\]

Combining this with step i, we have
\[d_H(\{\bW^{(k)}\}, \{\hat{\bW}_{\lambda}\}) = \max\left\{\sup_{\lambda} \inf_k \left\|\bW^{(k)} - \hat{\bW}_{\lambda}\right\|, \sup_{k} \inf_{\lambda} \left\|\bW^{(k)} - \hat{\bW}_{\lambda}\right\|\right\} \xrightarrow{(t, \epsilon) \to (1, 0)} 0\]
as desired.
\end{proof}

\section{\cbass: Algorithmic Regularization Paths for Convex Bi-Clustering} \label{sec:cbass}

Having explored the computational, theoretical, and practical advantages of \carp, we now turn to the closely related problem of bi-clustering. \emph{Bi-clustering} refers to the simultaneous clustering of rows and columns. Building on the convex clustering formulation \eqref{eqn:cclust}, \citet{Chi:2017} propose the following convex formulation of bi-clustering:
\begin{equation}
\hat{\bU}_{\lambda} = \argmin_{\bU \in \R^{n \times p}} \frac{1}{2} \|\bU - \bX\|_F^2 + \lambda \left(\sum_{\substack{i, j = 1 \\ i \neq j}}^n w_{ij} \|\bU_{i\cdot} - \bU_{j\cdot}\|_q + \sum_{\substack{k, l = 1 \\ k \neq l}}^p \tilde{w}_{kl} \|\bU_{\cdot k} - \bU_{\cdot l}\|_q\right).\label{eqn:bclust}
\end{equation}
The second penalty term induces row fusions, similarly to how the first term induces column fusions. The resulting $\hat{\bU}_{\lambda}$ has a `checkerboard' pattern where groups of rows and columns are clustered together. Note that for bi-clustering the centroids are \emph{scalars} instead of vectors as in the clustering case. 

Despite their relatively similar appearances, the convex bi-clustering problem \eqref{eqn:bclust} is significantly more complicated than the convex clustering problem \eqref{eqn:cclust} and cannot be directly solved directly using an operator splitting method. \citet{Chi:2017} propose the use of the \emph{Dykstra-Like Proximal Algorithm} (DLPA) of \citet{Bauschke:2008} to solve the convex bi-clustering problem \eqref{eqn:bclust} and refer to the resulting algorithm as \cobra (\textbf{Co}nvex \textbf{B}i-Cluste\textbf{R}ing \textbf{A}lgorithm). \cobra works by alternating solving row- and column-wise convex clustering problems until convergence. As with convex clustering, calculating the bi-clustering solution path with sufficient accuracy to accurately reconstruct both row and column dendrograms poses significant computational burden, which is exacerbated by \cobra's requirement to evaluate several convex clustering subproblems for each value of $\lambda$. While \carp could be used to solve each subproblem quickly, we would still have to run \carp many times, incurring a non-trivial total cost. 

Instead, we apply the technique of algorithmic regularization to \cobra directly: we take only a single DLPA step and, within that step, we take only a single ADMM step for each of the row- and column-subproblems. We refer to the resulting algorithm as \cbass--\textbf{C}onvex \textbf{B}i-clustering via \textbf{A}lgorithmic Regularization with \textbf{S}mall \textbf{S}teps. Details of the \cbass algorithm are given in Algorithm \ref{alg:cbass} below. Our \clustRviz software implements \cbass with and without a back-tracking step to ensure exact recovery or both the row- and column-dendrograms. 

\subsection{Algorithms for Convex Bi-Clustering}

The DLPA can be used to solve problems of
the form \[\prox_{(f + g)(\cdot)}(\br) = \argmin_{\bx} \frac{1}{2}\|\bx - \br\|_2^2 + f(\bx) + g(\bx)\]
where $f$ and $g$ are proximable but $f + g$ is not using the following iterative algorithm (see also Algorithm 10.18 in \citet{Combettes:2011}):

\begin{algorithm}[H]
\caption{DLPA: Dykstra-Like Proximal Algorithm}
\label{alg:dlpa-general}
\begin{enumerate}
\item Initialize: $\bx^{(0)} = \br$, $\bp^{(0)} = \bq^{(0)} = 0$, $k = 0$
\item Repeat until convergence:
\begin{itemize}
\item $\by = \prox_f(\bx^{(k)} + \bp^{(k)})$
\item $\bp^{(k+1)} = \bp^{(k)} + \bx^{(k)} - \by$
\item $\bx^{(k+1)} = \prox_g(\by^{(k+1)} + \bq^{(k)})$
\item $\bq^{(k + 1)} = \bq^{(k)} + \by - \bx^{(k + 1)}$
\item $k := k + 1$
\end{itemize}
\item Return $\bx^{(k)}$
\end{enumerate}
\end{algorithm}

To apply Algorithm \ref{alg:dlpa-general} to convex bi-clustering \eqref{eqn:bclust}, we note that the problem can be rewritten as: 
\begin{equation*}
\argmin_{\bU \in \R^{n \times p}} \frac{1}{2} \|\bU - \bX\|_F^2 + \underbrace{\lambda\left(\sum_{(e_l, w_l) \in \mathcal{E}_{\text{row}}} w_l  \|(\bD_{\text{row}}\bU)_{l\cdot}\|_q\right)}_{f(\bU) = P_{\text{row}}(\bU; \bw_{\text{row}}, q)} + \underbrace{\lambda\left(\sum_{(e_l, w_l) \in \mathcal{E}_{\text{col}}} w_l  \|(\bU\bD_{\text{col}})_{\cdot l}\|_q\right)}_{g(\bU) = P_{\text{col}}(\bU; \bw_{\text{col}}, q)}
\end{equation*}
and apply the DLPA with $f(\bU) = P_{\text{row}}(\bU; \bw_{\text{row}}, q)$ and $g(\bU) = P_{\text{col}}(\bU; \bw_{\text{col}}, q)$. We note that $\prox_f$ is a standard convex clustering problem and can be evaluated using the ADMM or AMA approaches described above. To evaluate $\prox_g$, we note that $\|(\bU\bD_{\text{col}})_{l\cdot}\|_q = \|(\bD_{\text{col}}^T\bU^T)_{\cdot l}\|_q$ so we simply need to perform standard convex clustering on transposed data. The DLPA then becomes: 

\begin{algorithm}[H]
\caption{DLPA for Convex Bi-Clustering}
\begin{enumerate}
\item Initialize: $\bU^{(0)} = \bX$, $\bP^{(0)} = \bQ^{(0)} = 0$, $k = 0$
\item Repeat until convergence:
\begin{itemize}
\item $\bT = \textsf{Convex-Clustering}(\bU^{(k)} + \bP^{(k)}; \mathcal{E}_{\text{row}})$
\item $\bP^{(k+1)} = \bP^{(k)} + \bX^{(k)} - \bT$
\item $\bU^{(k+1)} = \textsf{Convex-Clustering}((\bQ^{(k)} + \bT)^T; \mathcal{E}_{\text{col}})^T$
\item $\bQ^{(k + 1)} = \bQ^{(k)} + \bT - \bU^{(k + 1)}$
\item $k := k + 1$
\end{itemize}
\item Return $\bU^{(k)}$
\end{enumerate}
\end{algorithm}
\vspace{-0.2in}
Expanding the $\textsf{Convex-Clustering}$ steps with the ADMM from Algorithm \ref{alg:admm} yields Algorithm \ref{alg:dlpa}. To obtain \cbass from Algorithm \ref{alg:dlpa}, we replace the inner row- and column-subproblem loops with a single iteration of the convex clustering ADMM. Additionally, we do not reset the auxiliary $\bU, \bP, \bQ$ variables, instead carrying forward their values from each \cbass iteration to the next. These two modifications yield \cbass (Algorithm \ref{alg:cbass}).

\begin{algorithm}[p]
\footnotesize
\caption{Warm-Started DLPA + ADMM for the Convex Bi-Clustering Problem \eqref{eqn:bclust}}
\label{alg:dlpa}
\begin{enumerate}
\item Input:
\begin{itemize}
\item Data Matrix: $\bX \in \R^{n \times p}$
\item Weighted Directed Edge Sets: $\mathcal{E}_{\text{row}} = \{(e_l, w_l)\}$, $\mathcal{E}_{\text{col}} = \{(e_l, w_l)\}$
\item Relaxation Parameter: $ \rho \in \R_{> 0}$
\item Initial Regularization Parameter $\epsilon$ and Multiplicative Step-Size $t$
\end{itemize}
\item Precompute:
\begin{itemize}
\item Difference Matrices: $\bD_{\text{row}} \in \R^{|\mathcal{E}_{\text{row}}| \times n}$ and $\bD_{\text{col}} \in \R^{p \times |\mathcal{E}_{\text{col}}|}$
\item Cholesky Factors: $\bL_{\text{row}} = \textsf{chol}(\bI + \rho \bD_{\text{row}}^T\bD_{\text{row}}) \in \R^{n \times n}$ and $\bL_{\text{col}} = \textsf{chol}(\bI + \rho \bD_{\text{col}}\bD_{\text{row}}^T) \in \R^{p \times p}$
\end{itemize}
\item Initialize:
\begin{itemize}
\item $\bU^{(0)} = \bX$
\item $\bV^{(0)}_{\text{row}} = \bZ^{(0)}_{\text{row}}  = \bD_{\text{row}}\bX$
\item $\bV^{(0)}_{\text{col}} = \bZ^{(0)}_{\text{col}}  = (\bX\bD_{\text{col}})^T = \bD_{\text{col}}^T\bX^T$
\item $\bP^{(0)} = \bQ^{(0)} = 0$
\item $l = 0$, $\lambda_0 = \epsilon$, $k = 0$, 
\end{itemize}
\item Repeat until $\|\bV^{(k)}_{\text{row}}\| = \|\bV^{(k)}_{\text{col}}\|  = 0$:
\begin{itemize}
\item Repeat Until Convergence:
\begin{itemize}
\item Row Sub-Problem -- Repeat Until Convergence:
\begin{enumerate}
\item[(i)] $\bT = \bL_{\text{row}}^{-T}\bL^{-1}_{\text{row}}\left(\bU^{(k)} + \bP^{(k)} + \rho \bD_{\text{row}}^T(\bV_{\text{row}}^{(k)} - \bZ_{\text{row}}^{(k)}\right)$
\item[(ii)] $\bV^{(k+1)}_{\text{row}} = \prox_{\lambda_l / \rho\, P(\cdot; \bw_{\text{row}}, q)}\left(\bD_{\text{row}} \bT + \bZ^{(k)}_{\text{row}}\right)$
\item[(iii)] $\bZ^{(k + 1)}_{\text{row}} = \bZ^{(k)}_{\text{row}} + \bD \bT - \bV^{(k+1)}_{\text{row}}$
\end{enumerate}
\item $\bP^{(k+1)} = \bP^{(k)} + \bU^{(k - 1)} - \bT$
\item Column Sub-Problem -- Repeat Until Convergence:
\begin{enumerate}
\item[(i)] $\bS = \bL^{-T}_{\text{col}}\bL^{-1}\left((\bT + \bQ^{(k)})^T + \rho \bD_{\text{col}}(\bV^{(k)}_{\text{col}} - \bZ^{(k)}_{\text{col}}\right)$
\item[(ii)] $\bV^{(k+1)}_{\text{col}} = \prox_{\lambda_l / \rho\, P(\cdot; \bw_{\text{col}}, q)}\left(\bD^T_{\text{col}} \bS + \bZ^{(k)}_{\text{col}}\right)$
\item[(iii)] $\bZ^{(k + 1)}_{\text{col}} = \bZ^{(k)} + \bD_{\text{col}}^T \bS - \bV^{(k+1)}_{\text{col}}$
\end{enumerate}
\item $\bU^{(k+1)} = \bS^T$
\item $\bQ^{(k + 1)} = \bQ^{(k)} + \bT - \bU^{(k+1)}$
\item $k := k + 1$
\end{itemize}
\item Store $\hat{\bU}_{\lambda_l} = \bU^{(k)}$
\item Reset Auxiliary Variables: $\bU^{(k+1)} = \bX$, $\bP^{(k+1)} = \bQ^{(k+1)} = 0$
\item Update Regularization Parameter $\lambda_l := \lambda_{l - 1} * t$, $l := l + 1$
\end{itemize}
\item Return $\left\{\hat{\bU}_{\lambda_i}\right\}_{i = 0}^{l - 1}$ as the regularization path
\end{enumerate}
\end{algorithm}

\begin{algorithm}
\caption{\cbass: Convex Bi-Clustering via Algorithmic regularization with Small Steps}
\label{alg:cbass}
\begin{enumerate}
\item Input:
\begin{itemize}
\item Data Matrix: $\bX \in \R^{n \times p}$
\item Weighted Directed Edge Sets: $\mathcal{E}_{\text{row}} = \{(e_l, w_l)\}$, $\mathcal{E}_{\text{col}} = \{(e_l, w_l)\}$
\item Relaxation Parameter: $ \rho \in \R_{> 0}$
\item Initial Regularization Parameter $\epsilon$ and Multiplicative Step-Size $t$
\end{itemize}
\item Precompute:
\begin{itemize}
\item Difference Matrices: $\bD_{\text{row}} \in \R^{|\mathcal{E}_{\text{row}}| \times n}$ and $\bD_{\text{col}} \in \R^{p \times |\mathcal{E}_{\text{col}}|}$
\item Cholesky Factors: $\bL_{\text{row}} = \textsf{chol}(\bI + \rho \bD_{\text{row}}^T\bD_{\text{row}}) \in \R^{n \times n}$ and $\bL_{\text{col}} = \textsf{chol}(\bI + \rho \bD_{\text{col}}\bD_{\text{row}}^T) \in \R^{p \times p}$
\end{itemize}
\item Initialize:
\begin{itemize}
\item $\bU^{(0)} = \bX$
\item $\bV^{(0)}_{\text{row}} = \bZ^{(0)}_{\text{row}}  = \bD_{\text{row}}\bX$
\item $\bV^{(0)}_{\text{col}} = \bZ^{(0)}_{\text{col}}  = (\bX\bD_{\text{col}})^T = \bD_{\text{col}}^T\bX^T$
\item $\bP^{(0)} = \bQ^{(0)} = 0$
\item $k = 0$, $\gamma^{(0)} = \epsilon$ 
\end{itemize}
\item Repeat until $\|\bV^{(k)}_{\text{row}}\| = \|\bV^{(k)}_{\text{col}}\|  = 0$:
\begin{itemize}
\item Row Updates:
\begin{enumerate}
\item[(i)] $\bT = \bL_{\text{row}}^{-T}\bL^{-1}_{\text{row}}\left(\bU^{(k)} + \bP^{(k)} + \rho \bD_{\text{row}}^T(\bV_{\text{row}}^{(k)} - \bZ_{\text{row}}^{(k)}\right)$
\item[(ii)] $\bV^{(k+1)} = \prox_{\gamma^{(k)} / \rho\, P(\cdot; \bw_{\text{row}}, q)}\left(\bD_{\text{row}} \bT + \bZ^{(k)}\right)$
\item[(iii)] $\bZ^{(k + 1)}_{\text{row}} = \bZ^{(k)}_{\text{row}} + \bD_{\text{row}} \bT - \bV^{(k+1)}_{\text{row}}$
\end{enumerate}
\item $\bP^{(k+1)} = \bP^{(k)} + \bU^{(k - 1)} - \bT$
\item Column Updates:
\begin{enumerate}
\item[(i)] $\bS = \bL^{-T}_{\text{col}}\bL^{-1}\left((\bT + \bQ^{(k)})^T + \rho \bD_{\text{col}}(\bV^{(k)}_{\text{col}} - \bZ^{(k)}_{\text{col}}\right)$
\item[(ii)] $\bV^{(k+1)}_{\text{col}} = \prox_{\gamma^{(k)} / \rho\, P(\cdot; \bw_{\text{col}}, q)}\left(\bD^T_{\text{col}} \bS + \bZ^{(k)}_{\text{col}}\right)$
\item[(iii)] $\bZ^{(k + 1)}_{\text{col}} = \bZ^{(k)} + \bD_{\text{col}}^T \bS - \bV^{(k+1)}_{\text{col}}$
\end{enumerate}
\item $\bU^{(k+1)} = \bS^T$
\item $\bQ^{(k + 1)} = \bQ^{(k)} + \bT - \bU^{(k+1)}$
\item $k := k + 1$, $\gamma^{(k)} = \gamma^{(k - 1)} * t$
\end{itemize}
\item Return $\left\{\bU^{(k)}\right\}_{i = 0}^k$ as the \cbass algorithmic regularization path
\end{enumerate}
\end{algorithm}

\clearpage

\subsection{Visualizations for Convex Bi-Clustering}

While it is possible to construct row- and column-wise \cbass analogues of the \carp dendrogram and path plots discussed above, the primary visualization associated with bi-clustering is the \emph{cluster heatmap}, which combines a heatmap visualization of the raw data with independent row- and column-dendrograms \citep{Wilkinson:2009}. We modify the standard cluster heatmap by creating dendrograms using the fusions identified by \cbass. As \citet{Chi:2017} argue, the joint estimation of dendrograms provided by convex bi-clustering often produces better results than independent dendrogram construction. 

We applied \cbass to the \presidents data and show the resulting cluster heatmap in Figure \ref{fig:heatmap}. A close examination reveals several interesting patterns. This data clearly exhibits a bi-clustered structure, with certain words being strongly associated with certain groups of presidents. Examining the two clear bi-clusters on the left, we see that words such as ``billion,'' ``soviet,'' and ``technology'' are frequently used by modern presidents and rarely used by pre-modern presidents. Conversely, we see that words which may be considered somewhat antiquated, such as  ``vessel'' or ``shall,'' are associated with pre-modern presidents. For data with less clear structure, the interpretability of the cluster heatmap can sometimes be increased by plotting the smoothed estimates $\bU^{(k)}$ rather than the raw data.

\begin{figure}
  \centering
  \includegraphics[width=\textwidth]{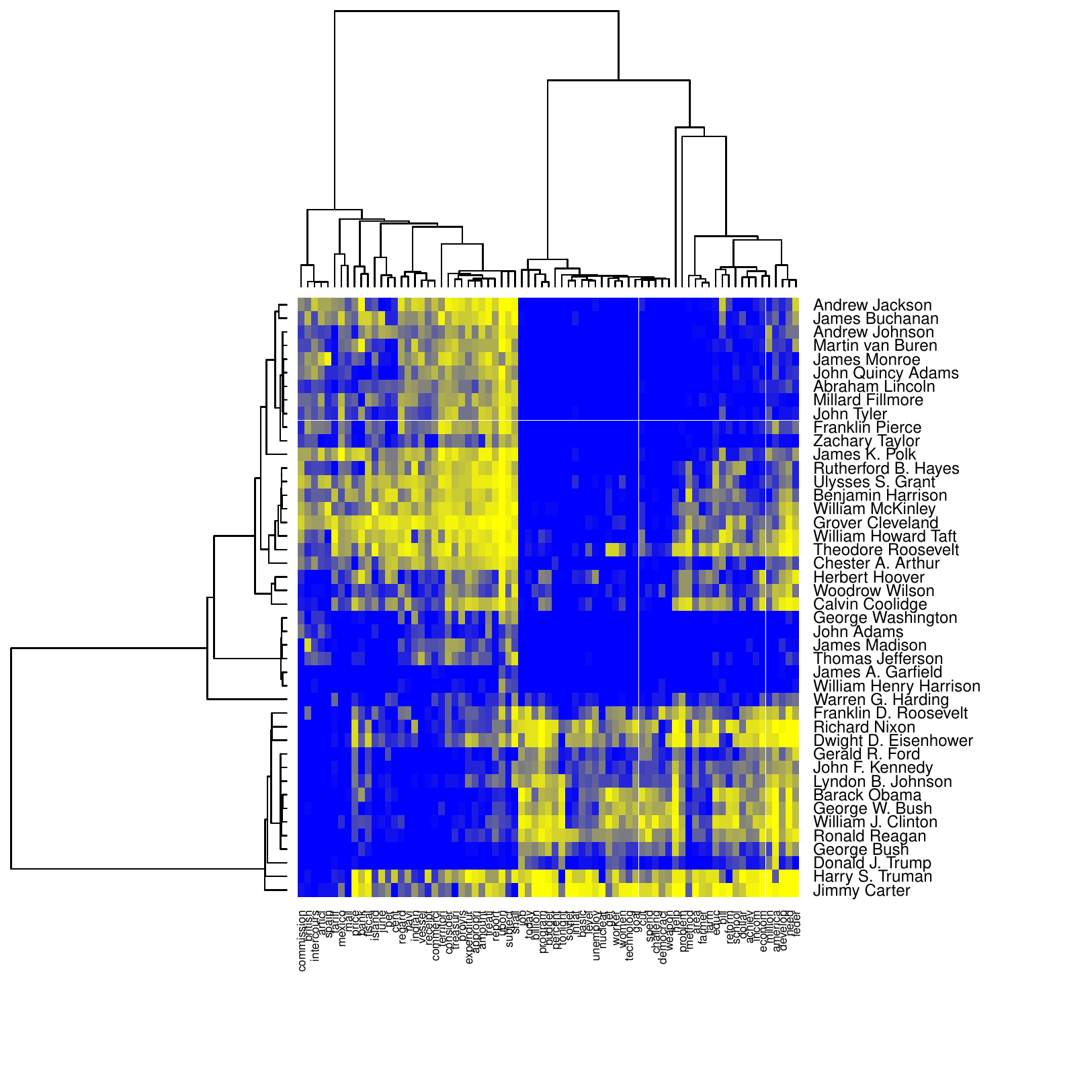}
  \caption{ Cluster heatmap of the \presidents data, with row and column dendrograms jointly estimated by \cbass. The partitions estimated by \cbass clearly associated modern words with modern presidents and old-fashioned words with pre-modern presidents.}
  \label{fig:heatmap}
\end{figure}

In simulation studies, \cbass appears to converge to the exact regularization path as $t \to 1$. While this is consistent with both our theory and observations for \carp, we leave the theoretical analysis of \cbass to future work. As far as we know, a rate of convergence has not been established for the DLPA in the optimization literature, without which the techniques used to prove Theorem \ref{thm:hausdorff} cannot be applied to \cbass. 

\section{Additional Comparisons}
\label{app:additional_comparisons}

Figure \ref{fig:accuracy_full} compares the accuracy of \carp, \cbass, hierarchical clustering, and $K$-means clustering on the \tcga and \authors data sets discussed in Section \ref{sec:comparisons}. While certain forms of hierarchical clustering perform well on this data, \carp achieves superior performance without requiring the user to select a distance or linkage. 

Figure \ref{fig:accuracy_gmm} compares the performance of \carp, hierarchical clustering with Euclidean distance and Ward's, complete, and single linkage, and $K$-means on data simulated from a Gaussian mixture model. The cluster centroids were equally spaced on a 2-dimensional subspace and $n = 54$ observations were generated from a Gaussian distribution with unit variance centered at the cluster centroid. Each of the clustering methods exhibit similar behaviors, with improved performance as the inter-cluster distance increases and decreased performance with higher ambient dimensionality or more clusters. Because these data were generated from isotropic Gaussians, all methods except single linkage hierarchical clustering perform well. 

Figure \ref{fig:accuracy_shapes} compares the performance of the same methods on  non-convex clusters. In particular, we consider a version of the ``half-moons'' example proposed by \citet{Hocking:2011}. (See also Figure \ref{fig:moons}.) The data were generated on a two-dimensional subspace with $n = 50$ observations from each cluster and Gaussian noise orthogonal to the signal subspace were added. Not surprisingly, the performance of all methods degrades as the degree of noise and the ambient dimensionality are increased. Despite this, we see that \carp and single-linkage hierarchical clustering clearly outperform other methods, with \carp being more robust to the presence of noise. 

Comparing these two simulations, we see that only convex clustering (\carp) is able to consistently perform well on both the convex and non-convex simulated data without requiring the user to select a distance metric or linkage. This is in large part due to the sparse weighting scheme used in the \clustRviz package, which is able to flexibly and robustly adapt to the observed data distribution. Our findings should be contrasted with those of \citet{Tan:2015} who focus only on the case of uniform weights and show that, without informative weights, convex clustering performs similarly to single linkage convex clustering. 

\begin{figure}
\centering
   \includegraphics[width=6in,height=6in]{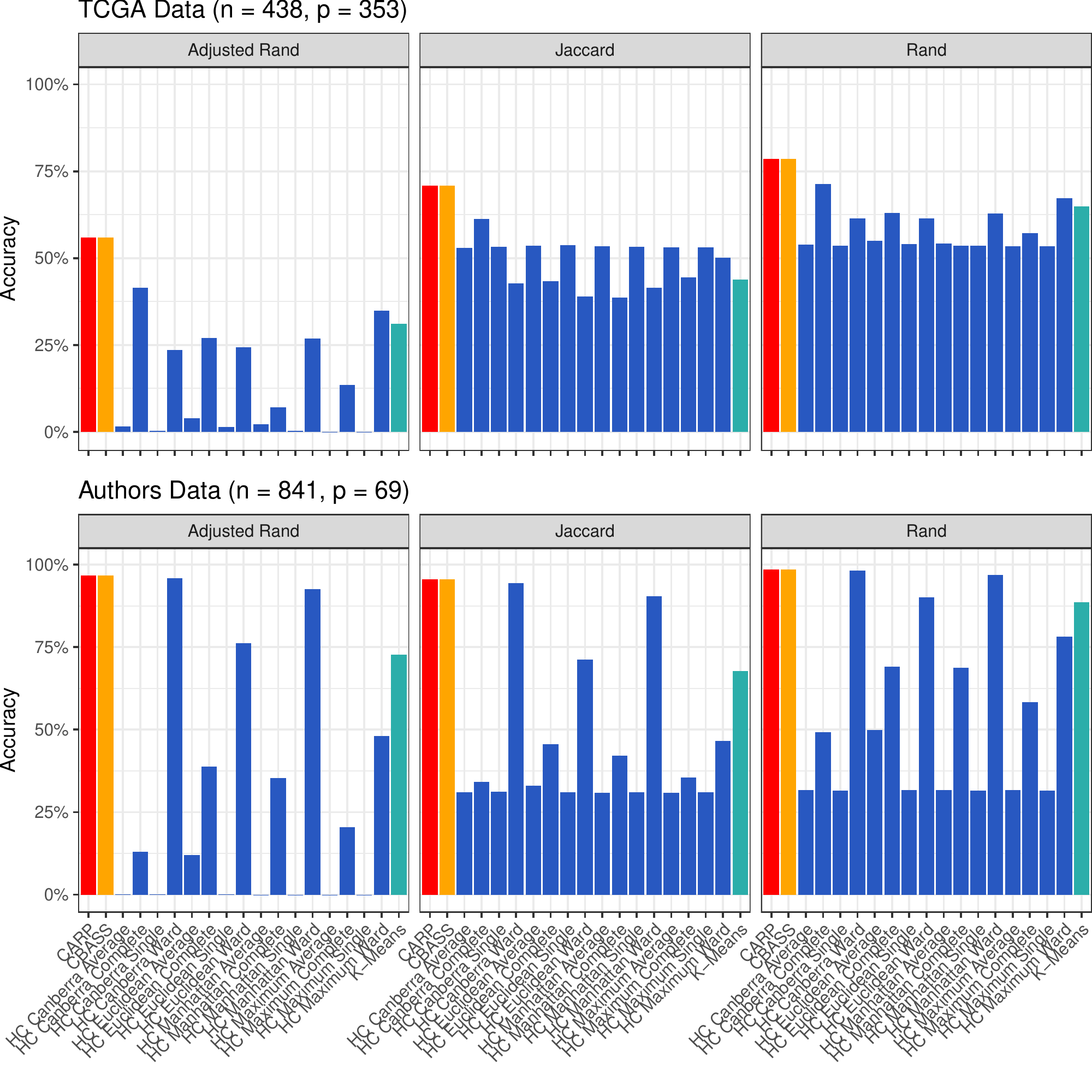}
    \caption{ Accuracy of convex clustering (\carp, red), convex bi-clustering (\cbass, yellow, discussed in Section \ref{sec:cbass}), Hierarchical Clustering (\texttt{HC}, blue), and $K$-means Clustering (teal) on the \textsf{TCGA} data set. \carp and \cbass consistently outperform both hierarchical and $K$-means clustering, as measured by the Rand \citep{Rand:1971}, Adjusted Rand \citep{Hubert:1985}, and Jaccard indices. \carp and \cbass were run using \clustRviz's default settings ($t = 1.05$ and $t = 1.01$ respectively).}
    \label{fig:accuracy_full}
\end{figure} 

\begin{figure}
\centering
   \includegraphics[width=6in,height=6in]{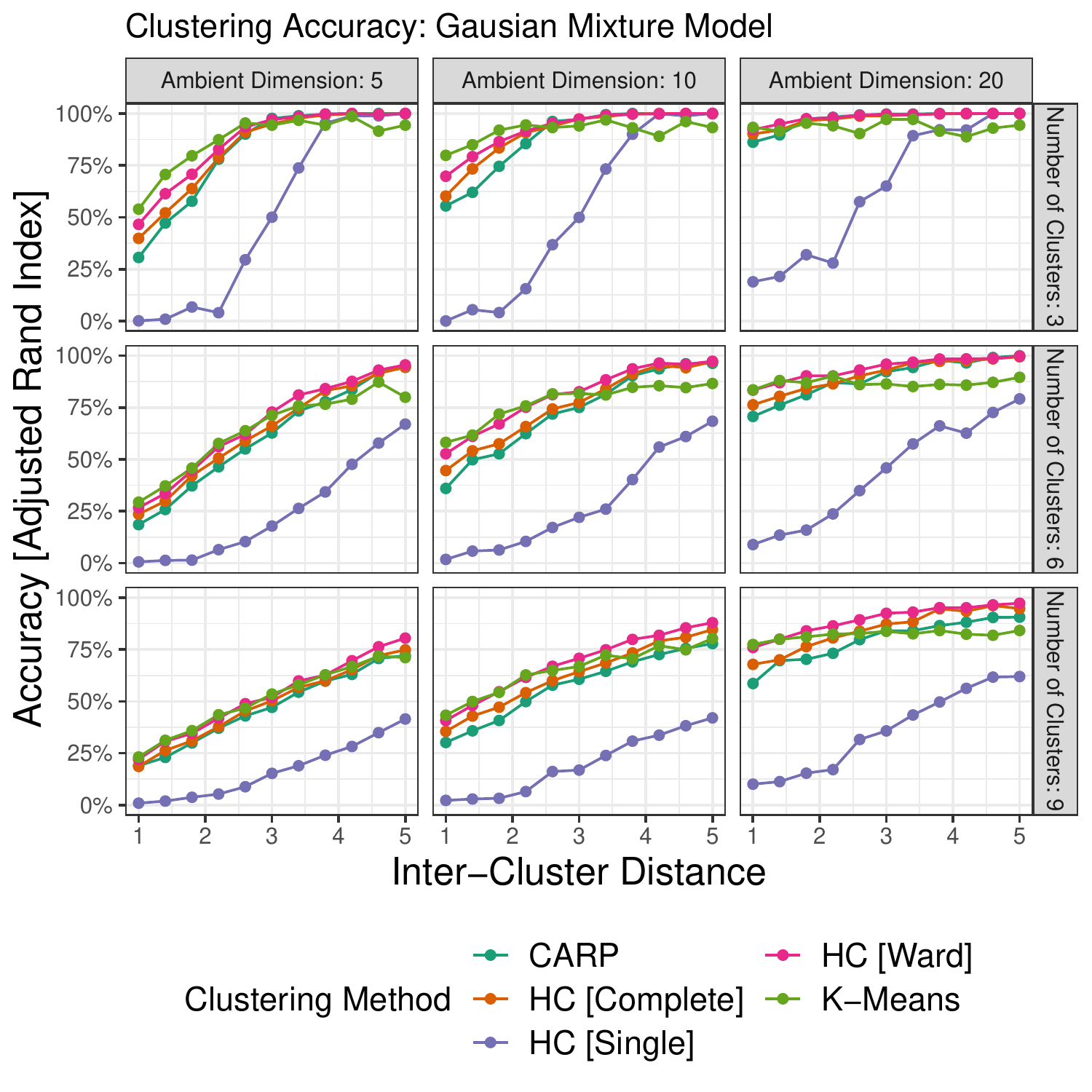}
    \caption{Accuracy of convex clustering, hierarchical clustering with Euclidean distance and several linkages, and $K$-means clustering (teal) on data simulated from a Gaussian mixture model, as measured by the Adjusted Rand \citep{Hubert:1985} index. Because these clusters are spherical, with sufficient inter-cluster separation all methods except hierarchical clustering with single linkage perform well.}
    \label{fig:accuracy_gmm}
\end{figure} 

\begin{figure}
\centering
   \includegraphics[width=6in,height=6in]{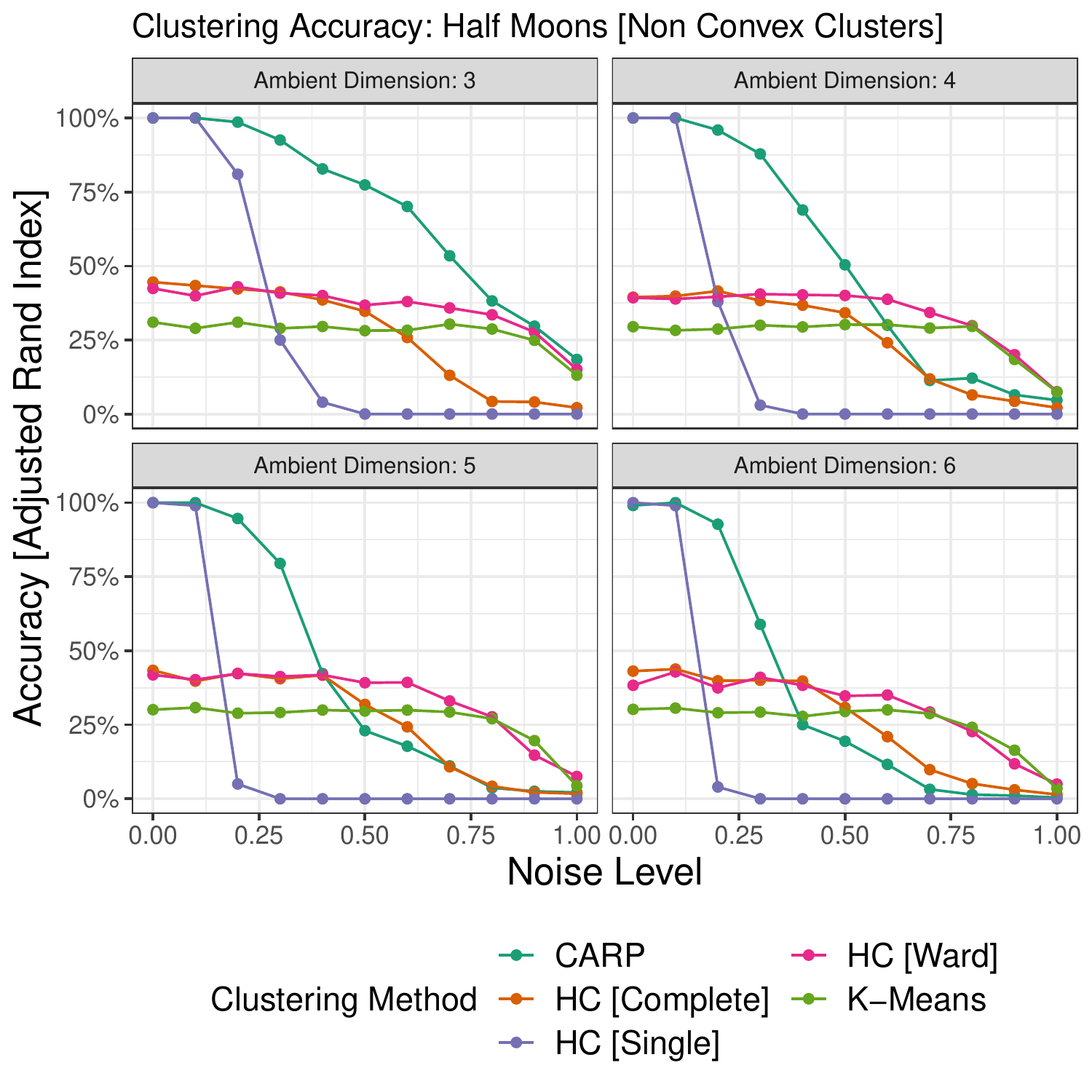}
    \caption{Accuracy of convex clustering, hierarchical clustering with Euclidean distance and several linkages, and $K$-means clustering (teal) on data simulated from a the two-circles and two-half-moons model, as measured by the Adjusted Rand \citep{Hubert:1985} index. Note that only \carp and single linkage hierarchical clustering are able to adapt to the non-convex cluster shapes.}
    \label{fig:accuracy_shapes}
\end{figure} 

\begin{figure}
\centering
   \includegraphics[width=3in,height=3in]{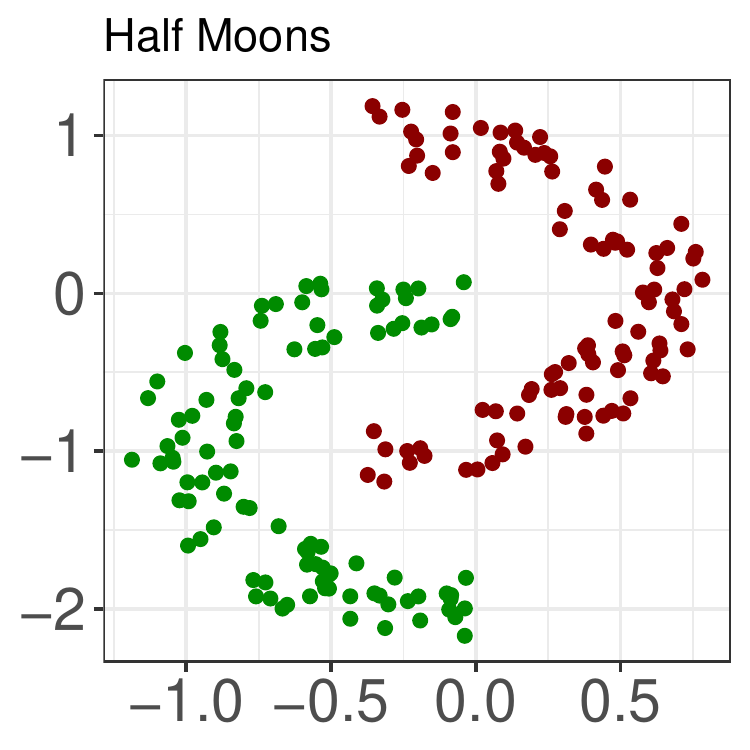}
    \caption{A sample realization of the interlocking half moons test data of \citet{Hocking:2011} used for Figure \ref{fig:accuracy_shapes}. Data were generated from these clusters along a random two-dimensional subspace and Gaussian noise orthogonal to the signal subspace was added to increase the difficulty of the clustering problem. \carp and single-linkage hierarchical clustering are able to exactly recover the true clustering in the noiseless case; the performance of hierarchical clustering quickly degrades as more noise is added, however, while \carp is more robust.}
    \label{fig:moons}
\end{figure}

\section{Back-Tracking, Post-Processing, and Dendrogram Construction}
\label{app:backtrack}
The \carpviz variant of our \carp algorithm implements a back-tracking scheme in order to improve dendrogram recovery. Because a relatively large value of $\lambda$ is typically required for any fusions to occur in convex clustering \eqref{eqn:cclust}, \carpviz begin with a large step-size (by default, $t = 1.1$) and performs standard \carp iterations until the first fusion is identified (\emph{i.e.}, a row of $\bV^{(k)}$ is set to zero). After the first fusion is identified, \carpviz switches to a smaller step-size (by default, $t = 1.01$) for the remainder of the algorithm. At each iteration, \carpviz counts the number of fusions that occur. If more than one fusion occurs, instead of proceeding, \carpviz attempts to determine which fusion occurred first. It does so using a back-tracking scheme, similar to those used in optimization methods. \carpviz discards the iteration with multiple fusions, halves the step-size, and performs another iteration. If this half-step iteration has only one fusion, \carpviz accepts it and continues as before. Otherwise, \carpviz again halves the step-size and repeats this process until the correct order of fusions is identified (or a limit on the number of back-tracking steps is hit). Once the first fusion is identified, \carpviz resets $t$ and continues. \cbassviz uses essentially the same scheme, though it checks for both row and column fusions. We have found that, because it only uses a small step-size at ``interesting'' parts of the solution space, this back-tracking scheme typically produces more accurate dendrogram recovery at less expense than running standard \carp with a very small step-size. 

Once \carp or \carpviz terminate, \clustRviz performs an additional post-processing step to isolate individual fusions. \clustRviz reviews the fusions at each iteration and, if an iteration has multiple fusions, linearly interpolates between $\bU^{(k)}$ and $\bU^{(k+1)}$ to determine the approximate regularization level at which each fusion occurred. The interpolated iterate is only approximate, but is necessary for dendrogram construction. We note that no interpolation is typically needed for \carpviz results, due to the back-tracking step used to isolate individual fusions, but, by default, \clustRviz post-processes both \carp and \carpviz output. The same post-processing scheme is applied separately to the row and column fusions from \cbass. 

Once post-processing is performed, a dendrogram is constructed from the interpolated iterates. The dendrogram construction proceeds in the opposite order as hierarchical clustering: we begin with the fully fused data and decrease $\gamma^{(k)}$, noting the order in which centroids were fused. (We use the reverse ordering so that, in the rare case where the path contains fissions, the \emph{final} fusion is reflected in the resulting dendrogram.) The dendrogram height associated with each fusion is the $\gamma^{(k)}$ at which that fusion is first observed. Finally, we check whether fusions are more uniformly distributed on the $\gamma^{(k)}$ scale or the $\log(\gamma^{(k)})$ and adjust the dendrogram height accordingly to provide less cluttered visualizations.

Since the weight selection, post-processing, and dendrogram reconstruction steps could potentially be applied to any convex clustering algorithm, they are omitted from all timing results shown in this paper.

\section{Additional Related Work}
\label{app:related}
Following its original introduction by \citet{Pelckmans:2005} and popularization by \citet{Hocking:2011} and \citet{Lindsten:2011}, convex clustering has been the subject of much methodological and theoretical research. In this section, we review some of this related work which, while not directly relevant to the computational or visualization strategies we propose, may be of interest to readers interested in convex clustering. 

The convex clustering problem can be generalized as \[\hat{\bU}_{\lambda} = \argmin_{\bU \in \R^{n \times p}} \frac{1}{2} \|\bX - \bU\|_F^2 + \lambda \sum_{\substack{i, j = 1 \\ i < j}}^n w_{ij} p\left(\bU_{i\cdot} - \bU_{j\cdot}\right)\] where $p(\cdot)$ is any sparsity-inducing function. The choice of an $\ell_q$-norm ($p(\cdot) = \|\cdot\|_q$) gives standard convex clustering as considered in this paper. \citet{Pan:2013}, \citet{Marchetti:2014}, \citet{Wu:2016}, and \citet{Shah:2017} have all considered the use of non-convex choices of $p(\cdot)$, typically using the popular SCAD or MCP penalty functions to reduce bias and improve estimation performance \citep{Fan:2001,Zhang:2010}. We do not consider non-convex $p(\cdot)$ in this paper, though the computational techniques and visualizations we propose could be adapted to non-convex penalties in a relatively straightforward manner.

Restricting our attention to standard convex clustering ($p(\cdot) = \|\cdot\|_q$), several useful methodological extensions have been proposed in the literature. For example, \citet{Wang:2016} augment the convex clustering problem \eqref{eqn:cclust} with an additional sparse component to add robustness to outliers, similar to the robust PCA formulation of \citet{Candes:2011}, while \citet{Wang:2018} propose a variant which incorporates feature selection into the clustering objective using an $\ell_1$ penalty \citep{Tibshirani:1996}. As discussed in Section \ref{sec:cbass}, \citet{Chi:2017} extend convex clustering to the bi-clustering setting, where rows and columns are simultaneously clustered. Building on this work, \citet{Chi:2018b} extend bi-clustering to general co-clustering of $k$-order tensors, where they note several surprising theoretical advantages. The recent paper by \citet{Park:2018} extends convex clustering to \emph{histogram-valued} data by replacing the Euclidean distance with an appropriate metric on the space of histograms.

The squared Frobenius loss function of the convex clustering problem may be interpreted as an isotropic Gaussian likelihood, suggesting another avenue for generalization. \citet{Sui:2018} replace the Frobenius loss with a squared Mahalanobis distance to improve performance on non-spherical clusters. If the metric (inverse covariance matrix) is known, simple variants on the techniques used in this paper may be used; if the metric must be estimated from the data, the resulting problem is bi-convex and an alternating minimization scheme must be used, only guaranteeing convergence to a stationary point. %Similarly, \citet{Wang:2019} propose replacing the Frobenius loss with one based on exponential family distributions or Bregman divergences, \red{MORE ABOUT THIS - also group shifted lasso stuff}.

The use of a convex formulation allows the sophisticated tools of modern high-dimensional statistics to be brought to bear \citep{Buhlmann:2011,Hastie:2015}. In addition to the work of \citet{Tan:2015} proving a form prediction consistency and of \citet{Radchenko:2017} proving asymptotic dendrogram recovery, \citet{Zhu:2014} give sufficient conditions for exact cluster recovery in the two-cluster case. The results of \citet{Zhu:2014} were later extended by \citet{Panahi:2017} and by \citet{Sun:2018} to the more general multi-cluster case. 

In addition to the general purpose operator-splitting algorithms proposed by \citet{Chi:2015}, specialized algorithms have been proposed for convex clustering in the ``large $n$'' (many observations) setting. \citet{Panahi:2017} propose a stochastic incremental algorithm based on the framework of \citet{Bertsekas:2011}, while \citet{Sun:2018} propose a semi-smooth Newton algorithm based on the framework of \citet{Li:2016}. \citet{Chen:2015} propose a proximal distance-based algorithm \citep{Lange:2014} and provide a GPU-based implementation. Recently, \citet{Ho:2019} proposed a generalized dual gradient ascent algorithm with linear convergence, though their approach only works for the $q = 1$ case; their approach is likely amenable to algorithmic regularization schemes similar to those we have propose for the ADMM.

The special case of convex clustering in $\R$ has been studied under various names, including \emph{total variation denoising} \citep{Rudin:1992}, the \emph{edge lasso} \citep{Sharpnack:2012} and the \emph{graph-fused lasso} \citep{Hoefling:2010}, or as a special case of the generalized lasso \citep{Tibshirani:2011}. When the underlying graph is a chain graph, convex clustering simplifies to the well-studied \emph{fused lasso} problem \citep{Tibshirani:2005,Rinaldo:2009,Johnson:2013}. 

Convex clustering has not yet seen significant adoption outside of the statistics and machine learning communities, though \citet{Chen:2015} discuss applications to human genomics. \citet{Nagorski:2018} propose an alternative weighting scheme based on genetic distances which they use to perform genomic region segmentation.

\printbibliography[title=Additional References,heading=subbibliography]
\end{refsection}

\end{document}